\documentclass[11pt]{article} 
\usepackage{times}
\usepackage{hyperref}
\usepackage{tikz}
\usepackage{amsmath,amssymb,amsfonts,array}
\usepackage{graphicx}
\usepackage{fullpage}
\graphicspath{{ExperimentResultCharts/}} 

\newcommand{\reals}{\mathrm{R}}

\renewcommand{\Pr}{\mathbb P} 

\newcommand{\mycomment}[1]{}
\newcommand{\comment}[1]{}

\usepackage[color=orange!20, linecolor=orange]{todonotes}
\newcommand{\abs}[1]{|#1|}
\newcommand{\E}{\mathrm{E}}
\newcommand{\Var}{\mathrm{Var}}
\newcommand{\sign}{\mathrm{sign}}
\def\norm#1{\left\| #1 \right\|}

\newtheorem{theorem}{Theorem}
\newtheorem{corollary}{Corollary}
\newtheorem{lemma}{Lemma}
\newtheorem{note}{Note}%

\newcommand{\ignore}[1]{}

\newcommand{\qed}{\mbox{\ \ \ }\rule{6pt}{7pt}}%
\newenvironment{proof}{\par{\bf Proof:}}{\qed \par}
\newenvironment{sketch}{\par{\sc Proof Sketch:}} {\qed \par}
\newenvironment{proofof}[1]{\par{\bf Proof (#1):}}{\qed \par}
\newenvironment{proofofof}[1]{\par{\bf Proof of #1:}}{\qed \par}
\newenvironment{program*}{\[\begin{array}{ >{\displaystyle}r >{\displaystyle}l >{\displaystyle}l }}{\end{array}\]\hspace{-1.4mm}}
\newenvironment{program}{\begin{equation}\begin{array}{ >{\displaystyle}r >{\displaystyle}l >{\displaystyle}l }}{\end{array}\end{equation}\hspace{-1.4mm}}

\author{Maria-Florina Balcan\thanks{School of Computer Science, Carnegie Mellon University.} \and 
Chris Berlind\thanks{College of Computing, Georgia Institute of Technology.} \and 
Avrim Blum\thanks{School of Computer Science, Carnegie Mellon University.}
\and Emma Cohen\thanks{School of Mathematics, Georgia Institute of Technology.}
\and Kaushik Patnaik\thanks{College of Computing, Georgia Institute of Technology.}
\and Le Song\thanks{College of Computing, Georgia Institute of Technology.}}

\title{Active Learning and Best-Response Dynamics}

\begin{document}
\maketitle

\begin{abstract}
We examine an important setting for engineered systems 
in which low-power distributed sensors are each making highly noisy
measurements of some unknown target function.  A center wants to accurately learn
this function by querying a small number of sensors, which
ordinarily would be impossible due to the high noise rate.
The question we address is whether local communication among sensors, together with natural best-response
dynamics in an appropriately-defined game, can denoise the system without destroying the true signal and allow the center
to succeed from only a small number of active queries.  By using techniques from game theory and empirical processes, we prove positive (and negative) results on the denoising power of
several natural dynamics.  We then show experimentally that when combined with recent agnostic active learning
algorithms, this process can achieve low error from very few queries, performing substantially better than active or passive learning 
without these denoising dynamics as well as passive learning with denoising.
\end{abstract}

\section{Introduction}
Active learning has been the subject of significant theoretical and
experimental study in machine learning, due to its potential to greatly
reduce the amount of labeling effort needed to learn a given target function.  However, to
date, such work has focused only on the single-agent low-noise
setting, with a learning algorithm obtaining labels from a single,
nearly-perfect labeling entity.  
In large part this is because the effectiveness of active learning is known to quickly 
degrade as noise rates become high~\cite{BeygelzimerDL:09}.
In this work, we introduce and analyze a novel setting where label information is held by highly-noisy
low-power agents (such as sensors or micro-robots).  We show how
by first using simple game-theoretic dynamics among the agents
we can quickly approximately denoise the system. This allows us to
exploit the power of active learning (especially, recent advances in
agnostic active learning), leading to
efficient learning from only a small number of expensive queries.

We specifically examine a setting relevant to many engineered systems 
where we have a large number of
low-power agents (e.g., sensors).  These agents are each measuring some quantity, such as whether
there is a high or low concentration of a dangerous chemical at their
location, but they are assumed to be highly noisy.  We also have a
center, far away from the region being monitored, which has the
ability to query these agents to determine their state.  Viewing the
agents as examples, and their states as noisy labels, the goal of the
center is to learn a good approximation to the true target function
(e.g., the true boundary of the high-concentration region for the
chemical being monitored) from a small number of label queries.
However,  because of the high noise rate, learning this function
directly would require a very large number of queries to be made
(for noise rate $\eta$, one would necessarily require  $\Omega(\frac{1}{(1/2 - \eta)^2})$ 
queries~\cite{BF}). 
The
question we address in this paper is to what extent this difficulty
can be alleviated by providing the agents the ability to engage in a
small amount of local communication among themselves.

What we show is that by using local communication and applying simple
robust state-changing rules such as following natural game-theoretic
dynamics, randomly distributed agents can modify their state in a way that greatly
de-noises the system without destroying the true target boundary.
This then nicely meshes with recent advances in agnostic active
learning~\cite{ABL14}, allowing for the center to learn a good approximation to the
target function from a small number of queries to the agents.  In particular,
in addition to proving theoretical guarantees on the denoising power of 
game-theoretic agent dynamics, we also show experimentally that
a version of the agnostic active learning algorithm of \cite{ABL14}, when
combined with these dynamics, indeed is able to achieve 
low error from a small number of queries, outperforming
active and passive learning algorithms without the best-response
denoising step, as well as outperforming passive learning 
algorithms with denoising.
More broadly, engineered systems such as sensor networks are especially well-suited to
active learning because components may be able to communicate among themselves
to reduce noise, and the designer has some control over how they are distributed
and so assumptions such as a uniform or other ``nice'' distribution
on data are reasonable. We focus
in this work primarily on the natural case of linear separator decision boundaries but 
many of our results extend directly to more general decision boundaries as well.

\subsection{Related Work}

There has been significant work in active learning (e.g., see \cite{Hanneke13,Settles12} and references therein), yet it is known 
active learning can provide significant benefits
in low noise scenarios only~\cite{BeygelzimerDL:09}.
There has also been extensive work analyzing the performance of simple dynamics in consensus games~\cite{Blume93,Ellison93,Morris00,KKT03,bbmdynamics,POUBBM}.
However this work has focused on getting to {\em some} equilibria or states of {\em low social cost}, while we are primarily
interested in getting near a specific desired configuration, which as we show below
is an approximate equilibrium.

\section{Setup}

We assume we have a large number $N$ of
agents (e.g., sensors) distributed uniformly at random in a geometric
region, which for concreteness we consider to be the unit ball in $R^d$.
There is an unknown linear separator such that in the initial state, each
sensor on the positive side of this separator is positive independently with probability $\geq 1-\eta$,
and each on the negative side is negative independently 
with probability $\geq 1-\eta$.  The quantity $\eta < 1/2$ is the {\em noise rate}.

\subsection{The basic sensor consensus game}
The sensors will denoise themselves by viewing themselves as
players in a certain consensus game, and performing a simple dynamics
in this game leading towards a specific $\epsilon$-equilibrium.

Specifically, the game is defined as follows, and is parameterised by
a communication radius $r$, which should be thought of as small.
Consider a graph where the sensors are vertices, and any two
sensors within distance $r$ are connected by an edge.  Each sensor is
in one of two states, positive or negative.  The {\em
payoff} a sensor receives is its correlation with its neighbors: the
fraction of neighbors in the same state as it minus the fraction in the
opposite state. So, if a sensor is in the same state as all its
neighbors then its payoff is 1, if it is in the opposite state of all
its neighbors then its payoff is $-1$, and if sensors are in uniformly
random states then the expected payoff is 0.  Note that the states of
highest social welfare (highest sum of utilities) are the all-positive
and all-negative states, which are {\em not} what we are looking for.
Instead, we want sensors to approach a different near-equilibrium
state in which (most of) those on the positive side of the target separator are positive
and (most of) those on the negative side of the target separator are negative.  For this
reason, we need to be particularly careful with the specific dynamics
followed by the sensors.

We begin with a simple lemma that for sufficiently large $N$, the
target function (i.e., all sensors on the positive side of the target separator in the
positive state and the rest in the negative
state) is an $\epsilon$-equilibrium, in that no sensor has more than
$\epsilon$ incentive to deviate.

\begin{lemma}\label{lemma:equilib}
For any $\epsilon,\delta>0$, for sufficiently large $N$, with probability $1-\delta$ the target function
is an $\epsilon$-equilibrium.
\end{lemma}
\begin{sketch}
The target function fails to be an $\epsilon$-equilibrium iff there exists a sensor for which more than an $\epsilon/2$
fraction of its neighbors lie on the opposite side of the separator.  Fix one sensor $x$ and
consider the probability this occurs to $x$, over the random placement of the $N-1$ other sensors.  Since the probability
mass of the $r$-ball around $x$ is at least $(r/2)^d$ (see discussion in proof of Theorem \ref{thm:basicupper}), so long as
$N-1 \geq (2/r)^d\cdot \max[8,\frac{4}{\epsilon^2}]\ln(\frac{2N}{\delta})$, with probability $1-\frac{\delta}{2N}$,
point $x$ will have $m_x \geq \frac{2}{\epsilon^2}\ln(\frac{2N}{\delta})$ neighbors (by
Chernoff bounds), each of which is at least as likely to be on $x$'s side of the target as on the other side.  Thus,
by Hoeffding bounds, the probability that more than a $\frac{1}{2}+\frac{\epsilon}{2}$ fraction lie
on the wrong side is at most $\frac{\delta}{2N} + \frac{\delta}{2N} = \frac{\delta}{N}$.  The result then follows
by union bound over all $N$ sensors.  For a bit tighter argument and a concrete bound on $N$,
see the proof of Theorem \ref{thm:basicupper} which
effectively has this as a special case.
\end{sketch}

Lemma \ref{lemma:equilib} motivates the use of best-response dynamics for denoising.  Specifically, we consider
a dynamics in which each sensor switches to the majority vote of all the other sensors in its neighborhood.
We analyze 
below the denoising power of this dynamics under both synchronous and asynchronous update models. 
In Appendix \ref{sec:conservative}, we also consider more robust (though less practical) dynamics in which 
sensors perform more involved computations over their neighborhoods.

\section{Analysis of the denoising dynamics}

\subsection{Simultaneous-move dynamics}
We start by providing a positive theoretical guarantee for one-round simultaneous move dynamics.
We will use the following standard concentration bound:
\begin{theorem}[Bernstein, 1924]
	Let $X = \sum_{i=1}^N X_i$ be a sum of independent random variables such that $\abs{X_i - \E[X_i]} \leq M$ 
for all $i$.  Then for any $t > 0$,
	$ \Pr[X - \E[X] > t] \leq \exp \left(\frac{-t^2}{2 (\Var[X] + Mt/3)}\right).$
\end{theorem}

\begin{theorem}\label{thm:basicupper}
If $N \geq \frac{2}{(r/2)^d (\tfrac{1}{2} - \eta)^2}\ln \left(\frac{1}{(r/2)^d (\tfrac{1}{2} - \eta)^2\delta}\right) +1,$
then with probability $\geq 1-\delta$, after one synchronous consensus update, every sensor at distance $\geq r$ from the separator has the correct label.
\end{theorem}

Note that since a band of width $2r$ about a linear separator has probability mass $O(r\sqrt{d})$,
Theorem \ref{thm:basicupper} implies that 
with high probability one synchronous update denoises all but an $O(r \sqrt{d})$ fraction of the sensors.  In fact, 
Theorem \ref{thm:basicupper} does not require the separator to be linear, and so this conclusion applies to any decision
boundary with similar surface area, such as an intersection of a constant number of halfspaces or a decision
surface of bounded curvature.

\begin{proofof}{Theorem \ref{thm:basicupper}}
Fix a point $x$ in the sample at distance $\geq r$ from the separator and consider the ball of radius $r$ centered at $x$. Let $n_+$ be the number of correctly labeled points within the ball and $n_-$ be the number of incorrectly labeled points within the ball. Now consider the random variable $\Delta = n_- - n_+$.  Denoising $x$ can give it the incorrect label  
only if $\Delta \geq 0$, so we would like to bound the probability that this happens. We can express $\Delta$ as the sum of $N-1$ independent random variables $\Delta_i$ taking on value 0 for points outside the ball around $x$, 1 for incorrectly labeled points inside the ball, or $-1$ for correct labels inside the ball.  Let $V$ be the measure of the ball centered at $x$ (which may be less than $r^d$ if $x$ is near the boundary of the unit ball).  Then since the ball lies entirely on one side of the separator we have
$$\E[\Delta_i] = (1-V)\cdot 0 + V\eta - V(1-\eta)= -V (1 - 2\eta).$$
Since $\abs{\Delta_i} \leq 1$ we can take $M = 2$ in Bernstein's theorem.  We can also calculate that $\Var[\Delta_i] \leq \E[\Delta_i^2] = V$.  Thus the probability that the point $x$ is updated incorrectly is
\begin{align*}
	\Pr\left[\sum_{i=1}^{N-1} \Delta_i \geq 0\right]
	&= \Pr\left[\sum_{i=1}^{N-1} \Delta_i - \E\Big[\sum_{i=1}^{N-1} \Delta_i\Big] \geq (N-1) V(1-2\eta)\right]\\
	&\leq \exp\left(\frac{-(N-1)^2 V^2(1-2\eta)^2}{2 \big((N-1) V + 2(N-1) V(1-2\eta)/3\big)}\right)\\
	&\leq \exp\left(\frac{-(N-1) V (1-2\eta)^2}{2 + 4(1-2\eta)/3}\right)\\
	&\leq \exp\left(-(N-1) V(\tfrac{1}{2}-\eta)^2\right)\\
	&\leq \exp\left(-(N-1) (r/2)^d(\tfrac{1}{2}-\eta)^2\right),
\end{align*}

where in the last step we lower bound the measure $V$ of the ball around $r$ by the measure of the sphere of radius $r/2$ 
inscribed in its intersection with the unit ball.
Taking a union bound over all $N$ points, it suffices to have
$ e^{-(N-1) (r/2)^d(\tfrac{1}{2}-\eta)^2} \leq \delta/N$,
or equivalently
$$ N-1 \geq \frac{1}{(r/2)^d (\tfrac{1}{2} - \eta)^2} \left(\ln N + \ln \frac{1}{\delta}\right).$$
Using the fact that $\ln x \leq \alpha x - \ln \alpha - 1$ for all $x,\alpha > 0$ yields the claimed bound on $N$.
\end{proofof}

We can now combine this result with the efficient
agnostic active learning algorithm of \cite{ABL14}.  In particular,
applying the most recent analysis of \cite{H13b,Y13} of the algorithm
of \cite{ABL14}, we get the following bound on the number of queries
needed to efficiently learn to accuracy  $1-\epsilon$  with probability
$1-\delta$.
\begin{corollary}
There exists constant $c_1>0$ such that for
$r\leq\epsilon/(c_1\sqrt{d})$, and $N$ satisfying the bound of Theorem \ref{thm:basicupper},
if sensors are each initially
in agreement with the target linear separator independently with probability at
least $1-\eta$, then one round of best-response dynamics is
sufficient such that the agnostic active learning algorithm of
\cite{ABL14} will efficiently learn to error $\epsilon$ using only
$O(d\log 1/\epsilon)$ queries to sensors.
\end{corollary}
In Section \ref{sec:experiments} we implement this algorithm and show that experimentally
it learns a low-error decision rule even in cases where the initial value of $\eta$ is quite high.

\subsection{A negative result for arbitrary-order asynchronous dynamics}
\label{sec:lower}
We contrast the above positive result with a negative result for arbitrary-order
asynchronous moves.  In particular, we show that for any $d\geq 1$, for sufficiently
large $N$,
with high probability there exists an update order that will cause all sensors
to become negative.

\begin{theorem}\label{thm:lower}
For some absolute constant $c>0$, if $r\leq 1/2$ and sensors begin with noise rate  $\eta$, and
$$N \geq \frac{16}{(cr)^d \phi^2}\bigg(\ln\frac{8}{(cr)^d\phi^2} + \ln \frac{1}{\delta}\bigg),$$
where $\phi = \phi(\eta) = \min(\eta, \tfrac{1}{2} - \eta)$, then with probability at least $1-\delta$ there exists an ordering of the agents so that asynchronous updates in this order cause all points to have the same label.
\end{theorem}
\begin{sketch}
Consider the case $d=1$ and a target function $x>0$.  Each subinterval of $[-1,1]$ of width $r$ has probability mass $r/2$, and let 
$m= r N/2$ be the expected number of points within such an interval.  The given value of $N$ is 
sufficiently large that with high probability, all such intervals in the initial state have both a positive count and
a negative count that are within $\pm \frac{\phi}{4}m$ of their expectations.  
This implies that if sensors update left-to-right, initially all sensors will (correctly) flip to negative, 
because their neighborhoods have more negative points than positive points.  But
then when the ``wave'' of sensors reaches the positive region, they will continue (incorrectly)
flipping to negative because the at least $m(1-\frac{\phi}{2})$ negative points in the left-half of their neighborhood will outweigh the at most $(1-\eta+\frac{\phi}{4})m$ positive points in the right-half of their neighborhood. 
For a detailed proof and the case of general $d>1$, see Appendix \ref{sec:proofs}.
\end{sketch}

\subsection{Random order dynamics}
While Theorem \ref{thm:lower} shows that there {\em exist} bad orderings for asynchronous
dynamics, we now show that we can get positive
theoretical guarantees for {\em random order}
best-response dynamics.

The high level idea of the analysis is to partition the sensors into three sets: those
that are within distance $r$ of the target separator, those at distance between
$r$ and $2r$ from the target separator, and then all the rest.  For those at distance
$<r$ from the separator we will make no guarantees: they might update incorrectly
when it is their turn to move due to their neighbors on the other side of the target.
Those at distance between $r$ and $2r$ from the separator
might also update incorrectly (due to ``corruption'' from neighbors at distance
$<r$ from the separator that had earlier updated incorrectly) but we will show that
with high probability this only happens in the last $1/4$ of the ordering.  I.e.,
within the first $3N/4$ updates, with high probability there are
no incorrect updates by sensors at distance between $r$ and $2r$ from the
target.  Finally, we show that with high probability, those at distance greater
than $2r$ {\em never} update incorrectly. This last part of the argument
follows from two facts: (1) with high probability all such points begin with more
correctly-labeled neighbors than incorrectly-labeled neighbors (so they will
update correctly so long as no neighbors have previously updated incorrectly),
and (2) after $3N/4$ total updates have been made, with high probability more than
half of the neighbors of each such point have already (correctly) updated, and
so those points will now update correctly no matter what their remaining neighbors do.
Our argument for the sensors
at distance in $[r,2r]$ requires $r$ to be small compared to
$(\frac{1}{2} - \eta)/\sqrt{d}$, and the final error is
$O(r\sqrt{d})$, so the conclusion is we have a total error less than $\epsilon$
for $r < c\min[\frac{1}{2}-\eta,\epsilon]/\sqrt{d}$ for some absolute constant $c$.

We begin with a key lemma.  For any given sensor, define its inside-neighbors to be its
neighbors in the direction of the target separator and its
outside-neighbors to be its neighbors away from the target
separator.  Also, let $\gamma = 1/2 - \eta$.

\begin{lemma}\label{lem:buffernodes}
For any $c_1, c_2>0$ there exist $c_3,c_4>0$ such that for $r \leq \frac{\gamma}{c_3\sqrt{d}}$
and $N \geq \frac{c_4}{(r/2)^d\gamma^2}\ln(\frac{1}{r^d\gamma\delta})$,
with probability $1-\delta$, each sensor $x$ at distance between $r$ and $2r$
from the target separator has $m_x \geq \frac{c_1}{\gamma^2}\ln(4N/\delta)$
neighbors, and furthermore the number of
inside-neighbors of $x$ that move before $x$ is within
$\pm\frac{\gamma}{c_2}m_x$ of the number of outside
neighbors of $x$ that move before $x$.
\end{lemma}

\begin{proof}
First, the guarantee on $m_x$ follows immediately from the fact that the probability
mass of the ball around each sensor $x$ is at least $(r/2)^d$, so for appropriate
$c_4$ the expected
value of $m_x$  is at least $\max[8,\frac{2c_1}{\gamma^2}]\ln(4N/\delta)$,
and then applying Hoeffding bounds \cite{Hoeffding63,BLM13} and the union bound.
Now, fix some sensor $x$ and let us first assume the
ball of radius $r$ about $x$ does not cross the unit sphere.
Because this is random-order dynamics, if $x$ is the $k$th sensor to move within
its neighborhood, the $k-1$ sensors that move earlier are each equally likely to be
an inside-neighbor or an outside-neighbor.  So the
question reduces to: if we flip $k-1 \leq m_x$ fair coins, what is the
probability that the number of heads differs from the number of tails
by more than $\frac{\gamma}{c_2} m_x$.  For $m_x \geq
2(\frac{c_2}{\gamma})^2\ln(4N/\delta)$, this is at most $\delta/(2N)$ by
Hoeffding bounds.   Now, if the ball of radius $r$ about $x$ does cross the unit sphere,
then a random neighbor is slightly more likely to be an inside-neighbor than an
outside-neighbor. However, because $x$ has distance at most $2r$ from the target
separator, this difference in probabilities is only $O(r\sqrt{d})$, which is at most
$\frac{\gamma}{2c_2}$ for appropriate choice of constant $c_3$.\footnote{We can
analyze the difference in probabilities as follows. First, in the worst case, $x$ is at distance exactly $2r$
from the separator, and is right on the edge of the unit ball.  So we can define our coordinate system to
view $x$ as being at location $(2r, \sqrt{1-4r^2},0, \ldots, 0)$.  Now, consider
adding to $x$ a random offset $y$ in the $r$-ball.  We want to look at the probability
that $x+y$ has Euclidean length less than 1 conditioned on the first coordinate of $y$
 being negative compared to this probability conditioned on the first coordinate of $y$ being
 positive.  Notice that because the second coordinate of $x$ is nearly 1,
 if $y_2 \leq -cr^2$ for appropriate $c$ then $x+y$ has length less than 1 no matter what
the other coordinates of $y$ are
(worst-case is if $y_1 = r$ but even that adds at most $O(r^2)$ to the squared-length).
On the other hand, if $y_2 \geq cr^2$ then $x+y$ has length greater than 1 also no matter
what the other coordinates of $y$ are.
So, it is only in between that the value of $y_1$ matters.  But notice that the
distribution over $y_2$ has maximum density $O(\sqrt{d}/r)$.  So, with probability nearly
$1/2$, the point is inside the unit ball for sure, with probability nearly $1/2$ the point is
outside the unit ball for sure, and only with probability $O(r^2\sqrt{d}/r) = O(r\sqrt{d})$
does the $y_1$ coordinate make any difference at all.}
So, the result follows by applying Hoeffding bounds to the $\frac{\gamma}{2c_2}$
gap that remains.
\end{proof}
\begin{theorem}
For some absolute constants $c_3,c_4$, for $r \leq \frac{\gamma}{c_3\sqrt{d}}$
and $N \geq \frac{c_4}{(r/2)^d\gamma^2}\ln(\frac{1}{r^d\gamma\delta})$,
in random order dynamics,
with probability $1-\delta$ all sensors at distance greater than $2r$ from
the target separator update correctly.
\end{theorem}
\begin{sketch}
We begin by using Lemma \ref{lem:buffernodes} to argue that with high probability, no points
at distance between $r$ and $2r$ from the separator update incorrectly within the first
$3N/4$ updates (which immediately implies that all points at distance greater than $2r$
update correctly as well, since by Theorem \ref{thm:basicupper}, with high probability they begin with more
correctly-labeled neighbors than incorrectly-labeled neighbors and their neighborhood only becomes more favorable).
In particular, for any given such point, the concern is that some of its inside-neighbors may have
previously updated incorrectly. However, we use two facts: (1) by Lemma \ref{lem:buffernodes}, we can set
$c_4$ so that
with high probability the total contribution of neighbors that have already updated is at most $\frac{\gamma}{8}m_x$
in the incorrect direction (since the outside-neighbors will have updated correctly, by induction),
and (2) by standard concentration inequalities \cite{Hoeffding63,BLM13},
with high probability at least $\frac{1}{8}m_x$ neighbors of $x$ have {\em not} yet updated.
These $\frac{1}{8}m_x$ un-updated neighbors together have in expectation a $\frac{\gamma}{4} m_x$ bias in the correct
direction, and so with high probability have greater than a $\frac{\gamma}{8}m_x$ correct bias for sufficiently large $m_x$
(sufficiently large $c_1$ in Lemma \ref{lem:buffernodes}).  So, with high probability this overcomes
the at most $\frac{\gamma}{8}m_x$ incorrect bias of neighbors that
have already updated, and so the points will indeed update correctly as desired.
Finally, we consider the points of distance $\geq 2r$.  Within
the first $\frac{3}{4}N$ updates, with high probability they will all update correctly as argued above.  Now consider time 
$\frac{3}{4}N$.
For each such point, in expectation $\frac{3}{4}$ of its neighbors have already updated, and with high probability,
for all such points the fraction of neighbors that have updated is more than half.  Since all neighbors
have updated correctly so far, this means these points will have more correct neighbors than incorrect neighbors no matter what the
remaining neighbors do, and so they will update correctly themselves.
\end{sketch}

\section{Query efficient polynomial time active learning algorithm}
\label{sec:active}

Recently, Awasthi et al.~\cite{ABL14} gave the first polynomial-time active learning algorithm
able to learn linear separators to error $\epsilon$ over the uniform
distribution in the presence of agnostic noise of rate $O(\epsilon)$.
Moreover, the algorithm does so with optimal query complexity of $O(d \log 1/\epsilon)$.
This algorithm is ideally suited to our setting because (a) the sensors are
uniformly distributed, and (b) the result of best response dynamics is 
noise that is low but potentially highly coupled (hence, fitting the low-noise
agnostic model).  In our experiments (Section \ref{sec:experiments})
we show that indeed this algorithm
when combined with best-response dynamics  achieves
low error from a small number of queries, outperforming
active and passive learning algorithms without the best-response
denoising step, as well as outperforming passive learning 
algorithms with denoising.

Here, we briefly describe the algorithm of \cite{ABL14} and the intuition
behind it.  At high level, the algorithm proceeds through several rounds, in each performing the following operations
(see also Figure \ref{fig:localization}):

\begin{description}
\item[Instance space localization:] 
Request labels for a random sample of points within a band of width $b_k = O(2^{-k})$ around the boundary of the previous
hypothesis $w_k$.
\end{description}

\begin{description}
\item[Concept space localization:] Solve for hypothesis vector $w_{k+1}$ by minimizing hinge loss subject to the 
constraint that $w_{k+1}$  lie within a radius $r_k$ from $w_k$; that is, $||w_{k+1} - w_k|| \leq r_k$.
\end{description}

\cite{ABL14,H13b,Y13} show that by setting the parameters appropriately (in particular, $b_k = \Theta(1/2^k)$ and $r_k = \Theta(1/2^k)$), the 
algorithm will achieve error $\epsilon$ using only $k=O(\log 1/\epsilon)$
rounds, with $O(d)$ label requests per round.
In particular, a key idea of their analysis is to decompose, in round $k$, the error of a candidate classifier $w$ as its error outside margin $b_{k}$ of the current separator plus its error inside margin $b_{k}$, and to prove that for these parameters, a small constant error inside the margin suffices to reduce overall error by a constant factor. A second key part is that by constraining the search for $w_{k+1}$ to vectors within a ball of radius $r_k$ about $w_{k}$, they show that hinge-loss acts as a sufficiently faithful proxy for 0-1 loss. 

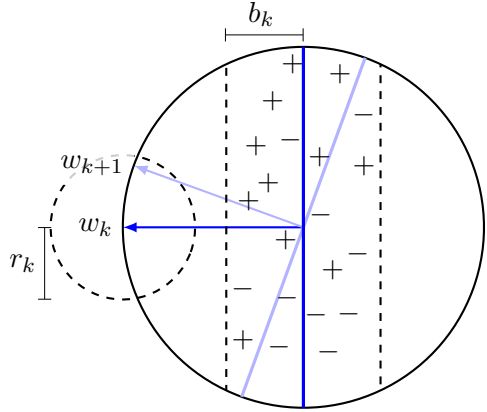
\begin{figure}
\begin{center}
\begin{tikzpicture}[thick, scale=.8, label distance=-.15cm, rotate=-20]
\usetikzlibrary{patterns}
\usetikzlibrary{shapes.geometric}

\draw[-latex, blue] (0,0) -- +(200:3cm) node[label={[black]180:{$w_{k}$}}] (w) {};
\draw[|-|, very thin] (110:3.2cm) -- ++(20:-.7cm) node[label=above:{$b_k$}] {} -- +(20:-.6cm);
\draw[dashed] (w) circle (1.2cm);
\draw[|-|, very thin] (200:4.3) -- ++(-70:.6) node[label=left:{$r_k$}] {} -- +(-70:.6);
\draw[-latex, blue!30] (0,0) -- (-3,0) node[label={[black,fill=white, opacity=.8, text opacity=1]180:{$w_{k+1}$}}] {};

\begin{scope}
	\clip[draw] (0,0) circle (3cm);
	\draw[very thick, blue!30] (0,-4) -- (0,4);
	\draw[very thick, blue] (-70:4cm) -- (110:4cm);
	\node (a) at (-70:.2cm) {};
	\draw[dashed] (1,1) +(-70:4cm) -- +(110:4cm);
	\draw[dashed] (-1,-1) +(-70:4cm) -- +(110:4cm);	

\end{scope}

\path
	(-1.2,1.8) node {$+$} to
	(-1.1,2.5) node {$+$} to
	(-.3,-2.1) node {$+$} to
	(-.6,-1.3) node {$-$} to
	(-.15,1.2) node {$+$} to
	(-.3,2.6) node {$+$} to
	(-1,.1) node {$+$} to
	(-1.2,1) node {$+$} to
	(-.8,.5) node {$+$} to
	(-.7,1.3) node {$-$} to
	(-.2, -.3) node {$+$};

\path
	(1.1,-1.8) node {$-$} to
	(.3,2.1) node {$-$} to
	(.6,1.3) node {$+$} to
	(.15,-1.2) node {$-$} to
	(.3,-2) node {$-$} to
	(1,-.1) node {$-$} to
	(1.2,-1.1) node {$-$} to
	(.7,-.5) node {$+$} to
	(.7,-1.3) node {$-$} to
	(.2, .3) node {$-$};
\end{tikzpicture}
\caption{The margin-based active learning algorithm after iteration $k$. The algorithm samples points within margin $b_k$ of the current weight vector $w_k$ and then minimizes the hinge loss over this sample subject to the constraint that the new weight vector $w_{k+1}$ is within distance $r_k$ from $w_k$.\label{fig:localization}}
\end{center}
\end{figure}

\section{Experiments}\label{sec:experiments}

In our experiments we seek to determine whether our overall algorithm of  best-response dynamics combined with active learning is effective at denoising the sensors and learning the target boundary. The experiments were run on synthetic data, and compared active and passive learning (with Support Vector Machines) both pre- and post-denoising.

\paragraph{Synthetic data.}
The $N$ sensor locations were generated from a uniform distribution over the unit ball in $\reals^2$, and the target boundary was fixed as a randomly chosen linear separator through the origin.
To simulate noisy scenarios, we corrupted the true sensor labels using two different methods: 1) flipping the sensor labels with probability $\eta$ and 2) flipping randomly chosen sensor labels and all their neighbors, to create pockets of noise, with $\eta$ fraction of total sensors corrupted.

\paragraph{Denoising via best-response dynamics.} 
In the denoising phase of the experiments, the sensors applied the basic majority consensus dynamic. That is, each sensor was made to update its label to the majority label of its neighbors within distance $r$ from its location\footnote{We also tested distance-weighted majority and randomized majority dynamics and experimentally observed similar results to those of the basic majority dynamic.}. We used radius values $r \in \{0.025, 0.05, 0.1, 0.2\}$.
Updates of sensor labels were carried out both through simultaneous updates to all the sensors in each iteration (synchronous updates) and updating one randomly chosen sensor in each iteration (asynchronous updates).

\paragraph{Learning the target boundary.}
After denoising the dataset, we employ the agnostic active learning algorithm of Awasthi et al.~\cite{ABL14} described in Section~\ref{sec:active} to decide which sensors to query and obtain a linear separator. We also extend the algorithm to the case of non-linear boundaries by implementing a kernelized version (see Appendix \ref{sec:moreexperiments} for more details). Here we compare the resulting error (as measured against the ``true'' labels given by the target separator) against that obtained by training a SVM on a randomly selected labeled sample of the sensors of the same size as the number of queries used by the active algorithm. We also compare these post-denoising errors with those of the active algorithm and SVM trained on the sensors before denoising. For the active algorithm, we used parameters asymptotically matching those given in Awasthi et al \cite{ABL14} for a uniform distribution. For SVM, we chose for each experiment the regularization parameter that resulted in the best performance.

\subsection{Results}

Here we report the results for $N = 10000$ and $r = 0.1$. Results for experiments with other values of the parameters are included in Appendix \ref{sec:moreexperiments}. Every value reported is an average over 50 independent trials.

\paragraph{Denoising effectiveness.}
Figure~\ref{fig:denoising} (left side) shows, for various initial noise rates, the fraction of sensors with incorrect labels after applying 100 rounds of synchronous denoising updates. In the random noise case, the final noise rate remains very small even for relatively high levels of initial noise. Pockets of noise appear to be more difficult to denoise. In this case, the final noise rate increases with initial noise rate, but is still nearly always smaller than the initial level of noise.

\begin{figure}[t]
\begin{center}
\includegraphics[width = 0.48\linewidth]{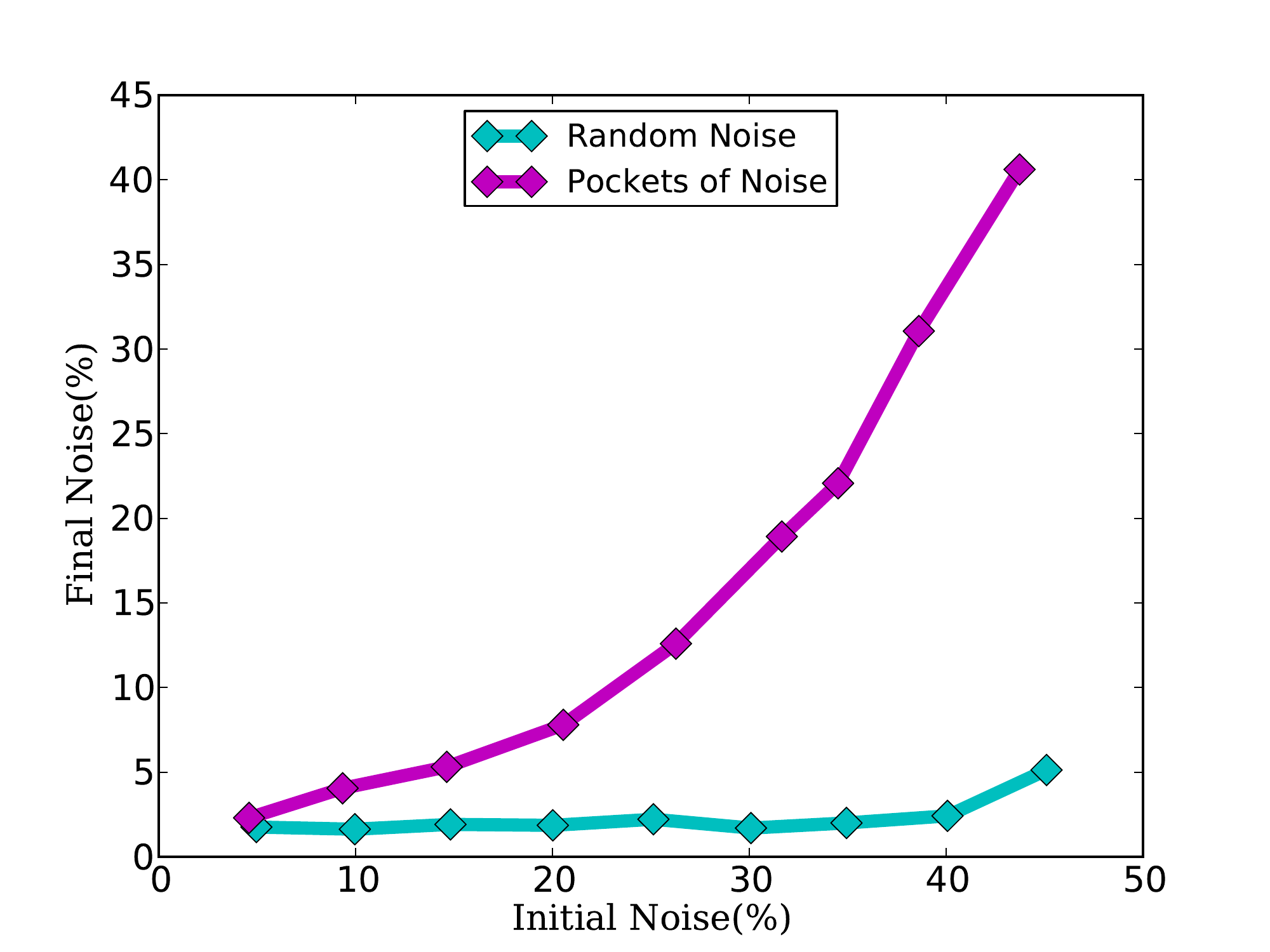}
\includegraphics[width = 0.48\linewidth]{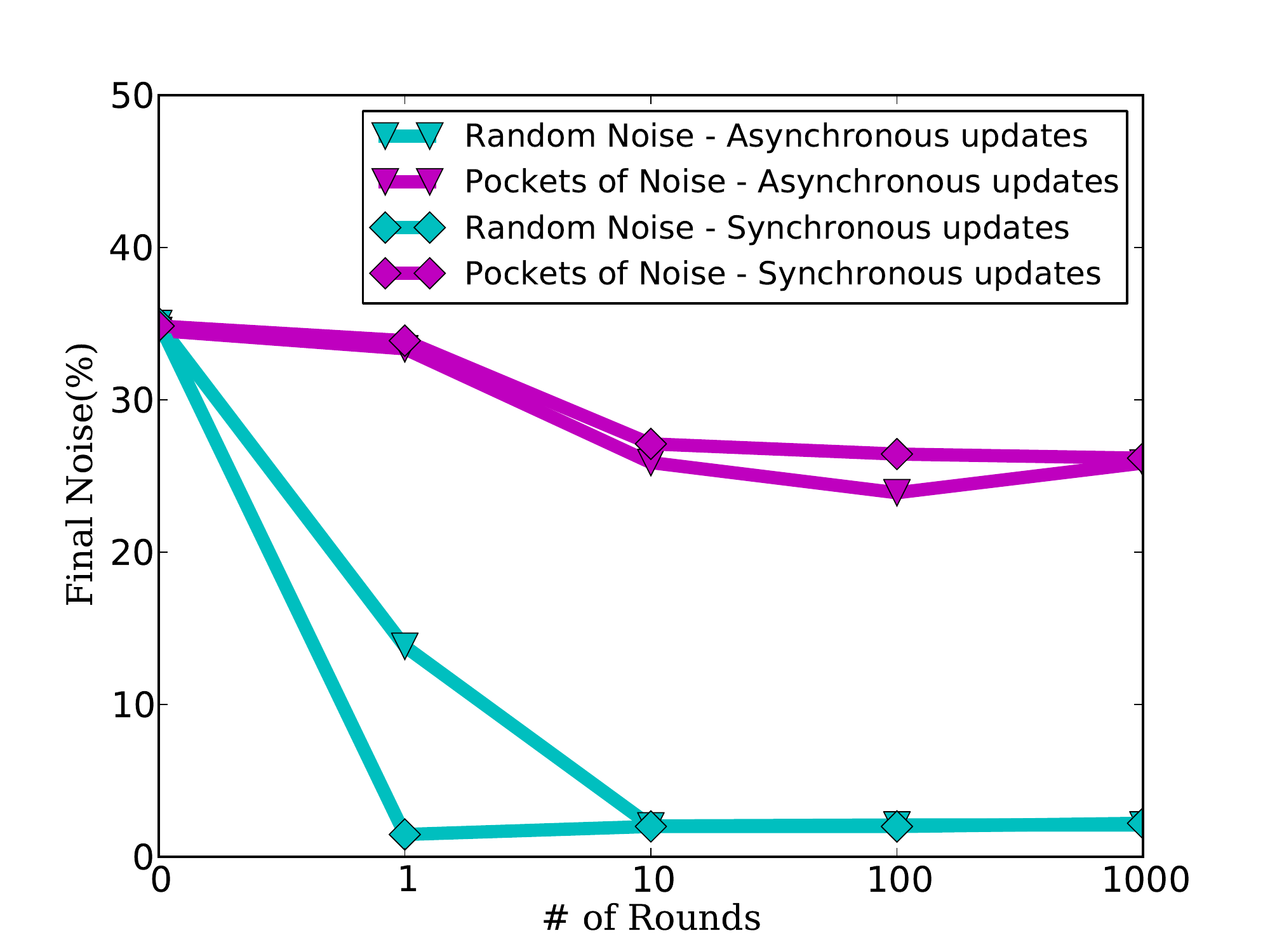}
\caption{Initial vs. final noise rates for synchronous updates (left) and comparison of synchronous and asynchronous dynamics (right). A synchronous round updates every sensor once simultaneously, while one asynchronous round consists of $N$ random updates.}
\label{fig:denoising}
\end{center}
\end{figure}

\paragraph{Synchronous vs.\ asynchronous updates.}

To compare synchronous and asynchronous updates we plot the noise rate as a function of the number of rounds of updates in Figure~\ref{fig:denoising} (right side). As our theory suggests, both simultaneous updates and asynchronous updates can quickly converge to a low level of noise in the random noise setting. Neither update strategy achieves the same level of performance in the case of pockets of noise.

\paragraph{Generalization error: pre- vs.~post-denoising and active vs.\ passive.}
We trained both active and passive learning algorithms on both pre- and post-denoised sensors at various label budgets,
and measured the resulting generalization error (determined by the angle between the target and the learned separator). The results of these experiments are shown in Figure~\ref{fig:error}. Notice that, as expected, denoising
helps significantly and on the denoised dataset the active algorithm achieves better generalization error than support vector machines at low label budgets.
For example, at a label budget of 30, active learning achieves generalization error approximately 33\% lower than the generalization error of SVMs.
Similar observations were also obtained upon comparing the kernelized versions of the two algorithms (see Appendix \ref{sec:moreexperiments}).

\begin{figure}[h]
\begin{center}
\includegraphics[width = 0.48\linewidth]{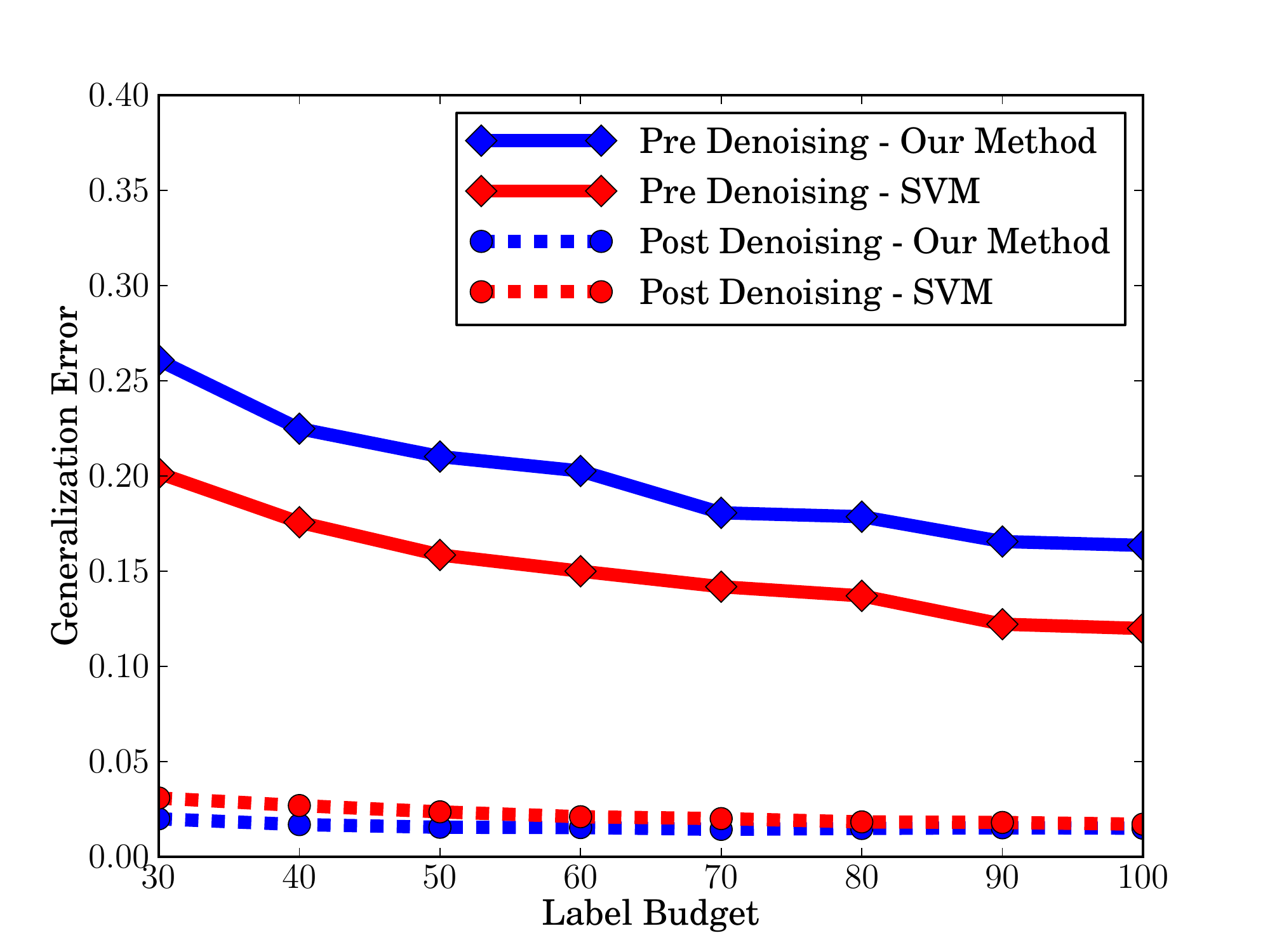}
\includegraphics[width = 0.48\linewidth]{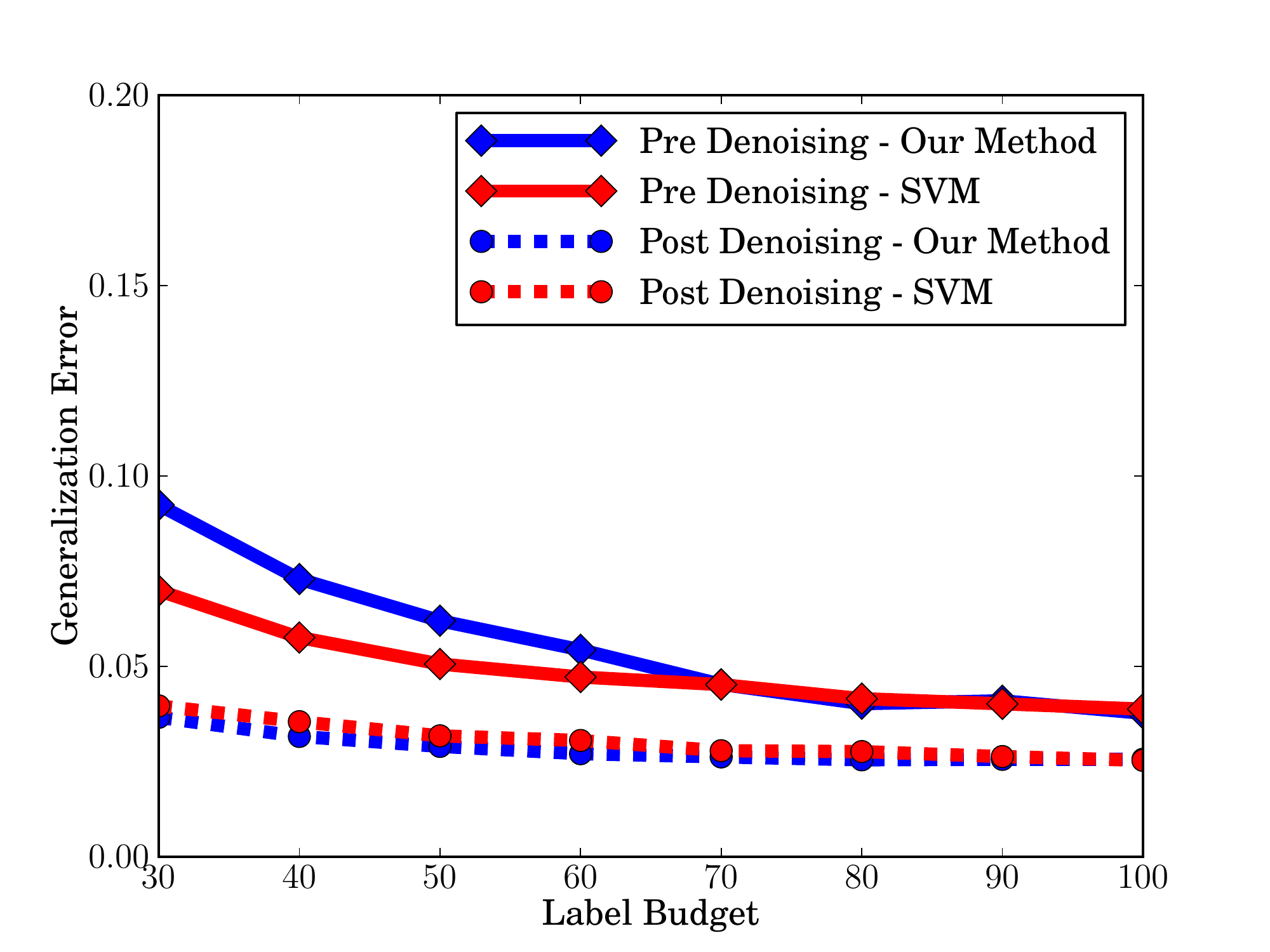}
\caption{Generalization error of the two learning methods with random noise at rate $\eta = 0.35$ (left) and pockets of noise at rate $\eta=0.15$ (right).}
\label{fig:error}
\end{center}
\end{figure}

\section{Discussion}
We demonstrate through theoretical analysis as well as experiments on synthetic data that local best-response
dynamics can significantly denoise a highly-noisy sensor network without destroying the underlying signal, allowing
for fast learning from a small number of label queries.   Another way to view this result is that the cost
function we {\em really} care about is that a sensor should get a cost of 1 for having a label that disagrees with the 
target function and a cost of 0 for a label that agrees with the target function, but unfortunately it cannot measure this
directly; so, instead we give each sensor a ``proxy'' objective that it {\em can} measure of agreeing with its neighbors.  
Our positive theoretical guarantees show that updating according to this proxy will perform well according to the true cost
function, and apply both to synchronous and random-order asynchronous updates.  This is borne out in
the experiments as well. Our negative result in Section \ref{sec:lower} for adversarial-order dynamics, in which
a left-to-right update order can cause the entire system to switch to a single label, raises
the question whether an alternative dynamics could be robust to adversarial update orders.  In 
Appendix \ref{sec:conservative} we present an alternative dynamics that we
prove is indeed robust to arbitrary update orders, but this dynamics is less practical because it requires substantially more
computational power on the part of the sensors.  It is an interesting question whether such general robustness can
be achieved by a simple practical update rule.  More generally, is there a different measurable, local, practical
proxy objective that would do an even better job than those considered here at optimizing the true error objective?    Another
natural direction is to explore the use of denoising protocols when tracking a target that is changing over time.

\subsection*{Acknowledgments}
This work was supported in part by ONR grant N00014-09-1-0751, NSF grants CCF-0953192,  CCF-1101283, and CCF-1101215, and AFOSR grant FA9550-09-1-0538.

\bibliographystyle{plain}
\bibliography{multi-agent-active}

\appendix

\section{Arbitrary Order and Conservative Best Response Dynamics}
\label{sec:conservative}
Given the negative result of Section \ref{sec:lower}, the basic best-response dynamics would
not be appropriate to use if no assumptions can be made about the
order in which sensors perform their updates.  To address this
problem, we describe here a modified dynamics that we
call {\em conservative} best-response.  The
idea of this dynamics is that sensors only change their state when
they are confident that they are not on the wrong side of the target
separator.  This dynamics is not as practical as regular best-response dynamics because
it requires substantially more computation on the part of the sensors.  Nonetheless,
it demonstrates that positive results for arbitrary update orders are indeed
achievable.
\begin{description}
\item[Conservative best-response dynamics:] In this dynamics,
sensors behave as follows:
\begin{enumerate}
\item If, for all linear separators through the sensor's location, a
majority of neighbors on both sides of the separator are positive, then flip to
positive.
\item If, for all linear separators through the sensor's location, a
majority of neighbors on both sides of the separator are negative, then flip to
negative.
\item Else (for some linear separator through the sensor's location,
the majority on one side is positive and the majority on the other
side is negative) then don't change.
\item To address sensors near the boundary of the unit sphere, in (1)-(3) we only consider
linear separators with at least $1/4$ of the points in their neighborhood on each side.
\end{enumerate}
\end{description}
\begin{theorem}\label{thm:safebr}
For absolute constants $c_3,c_4$, for $r \leq \frac{\gamma}{c_3\sqrt{d}}$
and $N \geq \frac{c_4}{(r/2)^d\gamma^2}\ln(\frac{1}{r^d\gamma\delta})$, in arbitrary-order
conservative best-response dynamics, each sensor whose $r$-ball does not intersect the target
separator will flip state correctly and no sensor will ever flip state incorrectly. 
\end{theorem}
Thus, Theorem \ref{thm:safebr} contrasts nicely with the negative result
in Theorem \ref{thm:lower} for standard best-response dynamics and
shows that the potential difficulties of arbitrary-order dynamics no
longer apply.

\begin{proof}
We will show that for the given value of $N$, with high probability we
have the following initial conditions: for each of the $N$ sensors,
for all hemispheres of radius $r$ centered at that sensor, the
empirical fraction of points in that hemisphere that are labeled
positive is within a $\gamma = 1/2 - \eta$ fraction of its
expectation.  This implies in particular that at the start of the
dynamics, all such hemispheres that are fully on the
positive side of the target separator have more positive-labeled
sensors than negative-labeled sensors, and all such hemispheres that
are fully on the negative side of the target separator have more negative-labeled
sensors than positive-labeled sensors.  This in turn implies, by
induction, that in the course of the dynamics, no sensor {\em ever}
switches to the incorrect label.  In particular, if we consider the
hyperplane through a sensor that is parallel to the target and consider
the hemisphere of neighbors on the ``good side'' of this hyperplane,
by induction none of those neighbors will have ever switched to the
incorrect label, and so their majority will remain the correct label,
and so the current sensor will not switch incorrectly by definition of
the dynamics.  In addition, it implies that all sensors whose $r$-balls
do not intersect the target separator {\em will} flip to the correct
label when it is their turn to update.

To argue this, for any fixed sensor and its neighborhood $r$-ball, since the 
VC-dimension of linear separators in $R^d$ is $d+1$, so long as we see
$$m \geq \frac{c}{\gamma^2}\left[d\ln(1/\gamma) + \ln(N/\delta)\right]$$
points inside that $r$-ball (for sufficiently large constant $c$), 
with probability at least $1-\delta/N$, for
any hyperplane through the center of the ball, the number of
positives in each halfspace will be within $\gamma m/8$ of the
expectation, and the number of negatives in each halfspace will be
within $\gamma m/8$ of the expectation.
This means that if the halfspace is fully on the positive side of the
target, then we see more positives than negatives, and if it is fully
on the negative side of the target then we see more negatives than
positives.  (In the case of sensors near the boundary of the unit
ball, this holds true for all with sufficiently many points, which includes the
halfspace defined by the hyperplane parallel to the target separator
if the sensor is within distance $r$ of the target separator
for $r < \frac{\gamma}{c_3\sqrt{d}}$.)
We are using $\delta/N$ as the failure probability so we
can do a union bound over all the $N$ balls.  Finally, we solve for
$N$ to ensure the above guarantee on $m$ to get
$$N \geq \frac{c_4}{(r/2)^d\gamma^2}\ln\left(\frac{1}{r^d\gamma\delta}\right)$$
points suffice for some constant $c_4$.
\end{proof}

\section{Additional Proofs}\label{sec:proofs}

\begin{proofofof}{Theorem \ref{thm:lower}}
Suppose the labeling is given by $\sign(w\cdot x)$.  We show that if sensors are updated in increasing order of $w\cdot x$ (from most negative to most positive) then with high probability all sensors will update to negative labels.

Consider what we see when we come to update the sensor at $x$.  Assuming we have not yet failed (given a positive label), all of the points $x'$ with $w\cdot x' < w\cdot x$ are labeled negative, while those with $w\cdot x' > w\cdot x$ are unchanged from their original states, and so are still labeled with independent uniform noise.  As in the proof of Theorem \ref{thm:basicupper}, we apply Bernstein's theorem to the difference $\Delta$ between the number of negative and positive points in the neighborhood of $x$, which we write as a sum of $(N-1)$ independent variables $\Delta_i$.  The expected labels of the nearby points depend on the location of $x$, so we consider three regions: $w\cdot x \leq -r$, $w\cdot x \geq 0$, and $-r < w\cdot x < 0$.

Let $V$ denote the probability mass of the ball of radius $r$ around $x$.  In all cases the variance is bounded by $\Var[\Delta_i] \leq \E[\Delta_i^2] = V \leq r^d$.

In the first region ($w\cdot x \leq -r$) we can use the same analysis from Theorem \ref{thm:basicupper} to find that $\E[\Delta_i] \leq -V(1-2\eta) \leq -(r/2)^d(1-2\eta)$, since the ball around $x$ never crosses the separator and any sensors previously updated to negative labels cannot hurt.

In the second region ($w\cdot x \leq 0$) we can use a similar analysis, bounding
$$\E[\Delta_i] \leq -V/2 + (1-\eta) V/2 = -\eta V/2 \leq -\tfrac{1}{2}(r/2)^d,$$
since the measure of the (positive biased) half of the ball further from the separator than $x$ is never larger than the measure of the remaining (all negative) half of the ball.

In the final region ($0 < w\cdot x < r$), we must take a little more care, as the measure of the all-negative half of the ball may be less than the measure of the unexamined side, which may be positive-biased due to crossing the separator.  To analyze this case, we project onto the 2-dimensional space spanned by $x$ and $w$.  The worst case is clearly when $x$ is on the surface of the ball, as shown in Figure \ref{fig:lower}.

\begin{figure}
\begin{center}
\begin{tikzpicture}[very thin, scale=.9]
\usetikzlibrary{patterns}

\clip[] (-5,0) rectangle (5,6.5);
\path (0,5) arc (90:100:5cm) node (c) {};
\node at (.2,.2) {$O$};
\node[color=red] at (-2,2) {$-$};
\node[color=blue] at (2,2) {$+$};
\node at (102:5.2) {$x$};
\node at (0.2,2.5) {$A$};
\path (c) +(-92:1.7) node {$B$};
\path (c) +(-54:1.8) node {$C$};

\begin{scope} 
\clip (0,0) circle (5cm);
\path[fill=red!50] (c) +(0,1.5) arc (90:270:1.5);
\end{scope}

\begin{scope} 
\clip (0,0) circle (5cm);
\path[fill=blue!15] (c) +(125:1.5cm) arc (125:-55:1.5);
\end{scope}

\begin{scope} 
\clip (0,0) circle (5cm);
\clip (0,0) rectangle (1,7);
\path[fill=blue!50] (c) +(0,1.5) arc (90:-90:1.5);
\end{scope}


\draw (0,0) circle (5cm); 
\draw[thick] (5,0) -- (-5,0);
\draw (0,0) -- (0, 7); 

\draw (c) circle (1.5cm); 
\draw (c) +(0,-1.5) -- +(0,1.5); 

\draw (0,0) -- (c.center);
\draw[thick] (c.center) -- +(-55:1.5cm); 
\begin{scope}
\clip(0,0) rectangle (-3,7);
\draw[thick] (c.center) -- +(-70:3cm); 
\end{scope}

\draw (c.center) -- +(199:1.5cm);
\draw (c.center) -- +(1:1.5cm);

\draw (c.center) +(199:.4cm) arc (199:270:.4cm); 
\path (c) node at +(235:.65) {$\beta$} node at +(-28:.6) {$\alpha$};
\draw (c.center) +(-55:.4cm) arc (-55:1:.4cm); 
\draw (0,1) arc (90:100:1);
\node at (95:1.2) {$\theta$};

\end{tikzpicture}
\caption{A ball around $x$ intersecting the decision boundary and the boundary of the unit ball.\label{fig:lower}}
\end{center}
\end{figure}
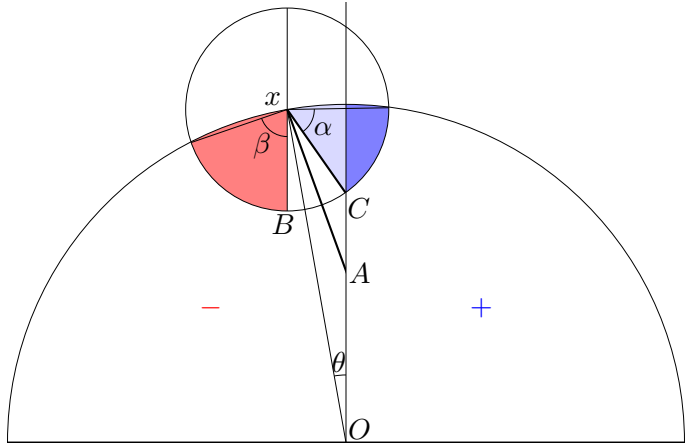

Any point in the red region is known to have a negative label, while points in the dark blue region are biased towards positive labels.  We first show that the red region is bigger by showing that the angle $\alpha$ subtended by the dark blue region is smaller than the angle $\beta$ of the red region.  Construct the segment $\overline{xA}$ by reflecting the segment $\overline{xB}$ about the line $\overline{xO}$ and extending it to the separator.  Note that the angle $\angle OxA$ is the same as the angle $\theta$ between $x$ and the separator.  We find that $\alpha \leq \beta$ precisely when $xA \geq xC = r$.  Indeed, by considering the isosceles triangle $\triangle AxO$ we see that $xA = 1/(2\cos\theta)\geq 1/2$.  So as long as $r\leq 1/2$ we have $\beta \geq \alpha$.
Thus, since the projection of the uniform distribution over the unit ball onto this plane is radially symmetric, the red region
has more probability mass than the blue region.

We can now calculate for this case
\begin{align*}
	\E[\Delta_i] &\leq (-1)[\text{measure of red}] + (1 - 2\eta)[\text{measure of blue}] + (2\eta - 1)[\text{measure of white}]\\
	&\leq -2\eta[\text{measure of red}].
\end{align*}
Note that although the projection does not make sense for $d=1$ the result obviously still holds (as there are no points near both the separator and the boundary of the unit ball). We can lower bound the measure of the red region by the measure of the sphere inscribed in the sector, which has radius at least $cr$ for some $0 < c < 1/2$ as long as $r \leq 1/2$ (since $\beta$ is bounded away from 0 in this range of $r$).

Now we see that for any $x$ the expected label satisfies
$$ \E[\Delta_i] \leq -\tfrac{1}{2} (cr)^d \min(\eta, \tfrac{1}{2} - \eta).$$
Letting $\phi = \min(\eta, \frac{1}{2} - \eta)$, we find that the probability of giving a positive label on any given update is
\begin{align*}
  \Pr[\Delta \geq 0] 
  &\leq \exp\left(\frac{-\frac{1}{4}(N-1)^2 (cr)^{2d} \phi^2 / 2}{(N-1) r^d + (N-1) (cr)^d \phi / 3}\right)\\
  &= \exp\left(\frac{-\frac{1}{4} (N-1) (cr)^{d} \phi^2}{1 + \phi / 3}\right)\\
  &= \exp\left(-(N-1)(cr)^d \phi^2/8\right)
\end{align*}
By the union bound, we find that
$$ N \geq \frac{16}{(cr)^d \phi^2}\bigg(\ln\frac{8}{(cr)^d\phi^2} + \ln \frac{1}{\delta}\bigg) $$
suffices to ensure that with probability at least $1-\delta$ all sensors are updated to negative labels.
\end{proofofof}

\begin{note} If $r = O(1/\sqrt{d})$ then we can lower bound all of the relevant measures in the preceding proof by $\Theta(r^d)$ rather than $(\Theta(r))^d$, to see that
\[ N \geq \Omega\left(\frac{1}{r^d \phi^2} \left(\ln \frac{1}{r \phi} + \ln \frac{1}{\delta}\right)\right)\]
suffices.
\end{note}

\section{Additional Experimental Results}\label{sec:moreexperiments}

All of the following experiments were run with initial noise rate $\eta = 0.35$ for random noise and $\eta = 0.15$ for pockets of noise, and the results have been averaged over 20 trials of 50 iterations each.

\textbf{Effect of number of sensors on denoising and learning.}

We analyze the performance of learning post-denoising as a function of the number of sensors for a fixed radius.  Given the results of Theorem 2 in Section 3.1 for synchronous updates, we expect the denoising to improve as sensors are added, which in turn should improve the generalization error post-denoising.  Figures 5, 6, and 7 show the generalization error pre- and post-denoising for $N\in\{1000, 5000, 25000\}$.  For a budget of 30 labels on random noise, the noise rate after denoising drops from 12.0\% with 1000 sensors to 1.7\% with 25000 sensors, and with this improvement we see a corresponding drop in generalization error from 7.4\% to 1.6\%.  Notice that denoising helps for both active and passive learning in all scenarios except for the case of pockets of noise with 1000 sensors, where the sensor density is insufficient for the denoising process to have significant effect.

\begin{figure}
\begin{center}
\includegraphics[width = 0.48\linewidth]{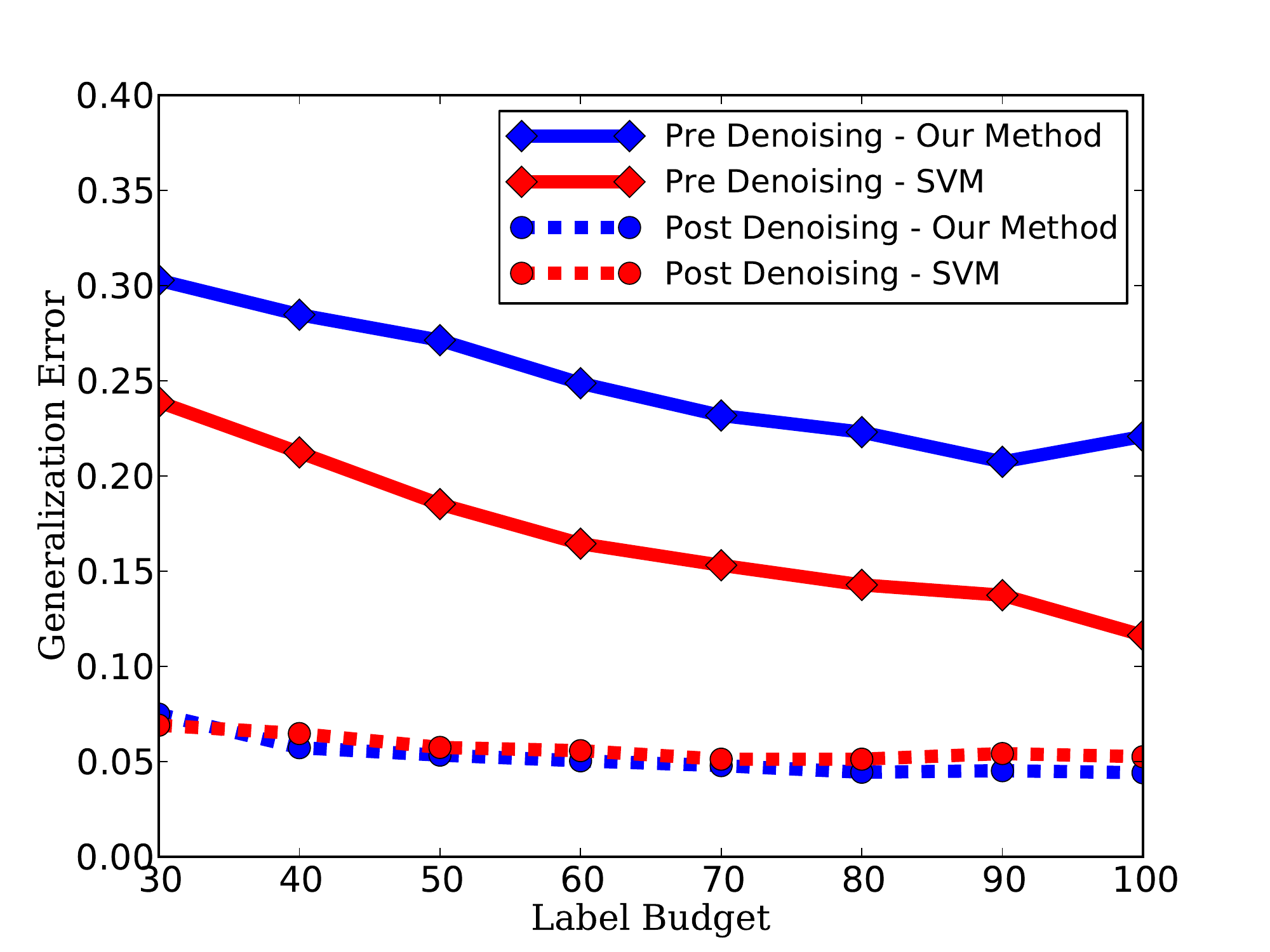}
\includegraphics[width = 0.48\linewidth]{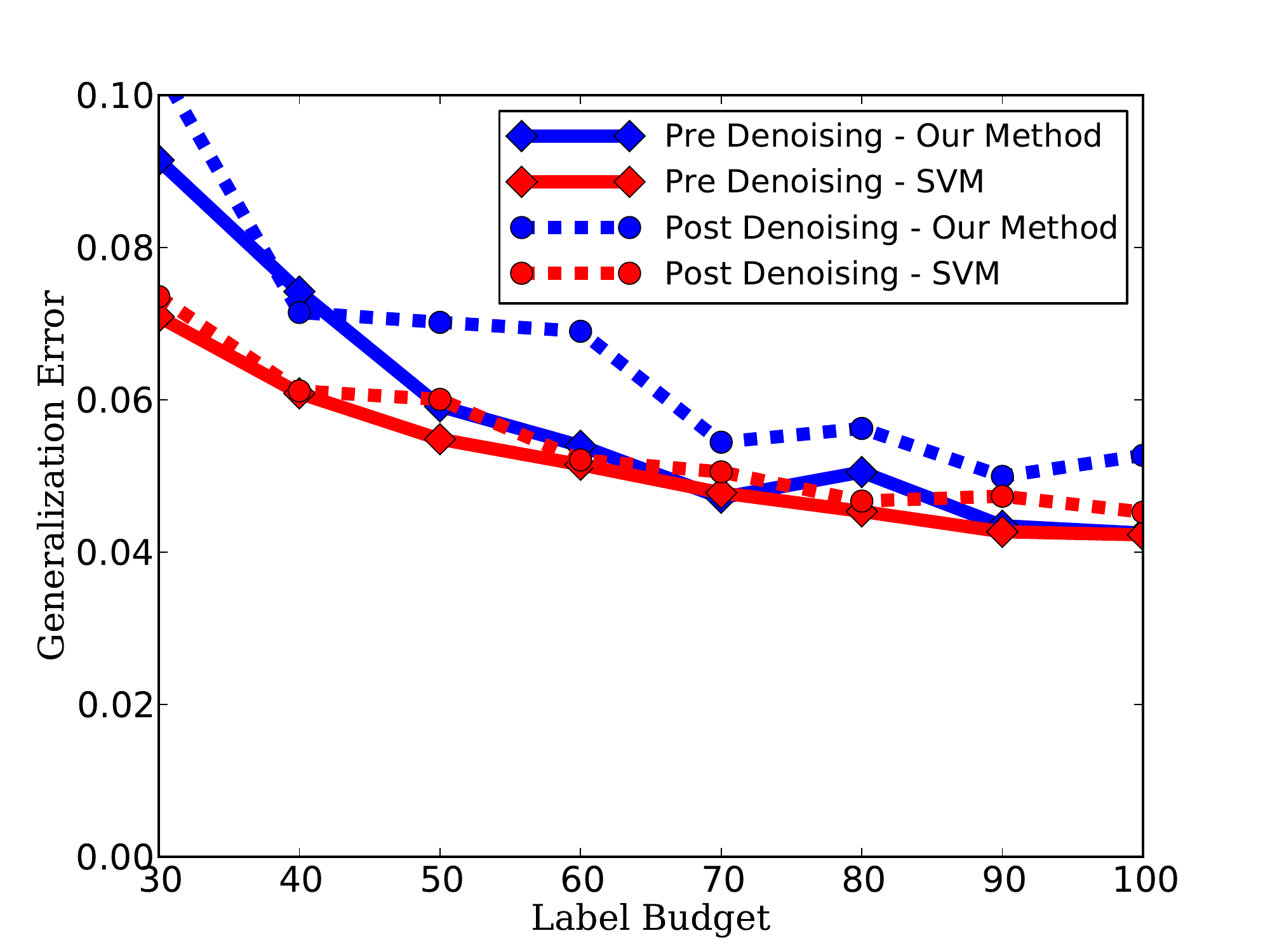}
\caption{Generalization error with 1000 sensors in different noise scenarios. Generalization error in y-axis and Labels used in x-axis. From Left to Right - Random Noise, and Pockets of Noise.\label{fig:1000}}
\end{center}
\end{figure}

\begin{figure}
\begin{center}
\includegraphics[width = 0.48\linewidth]{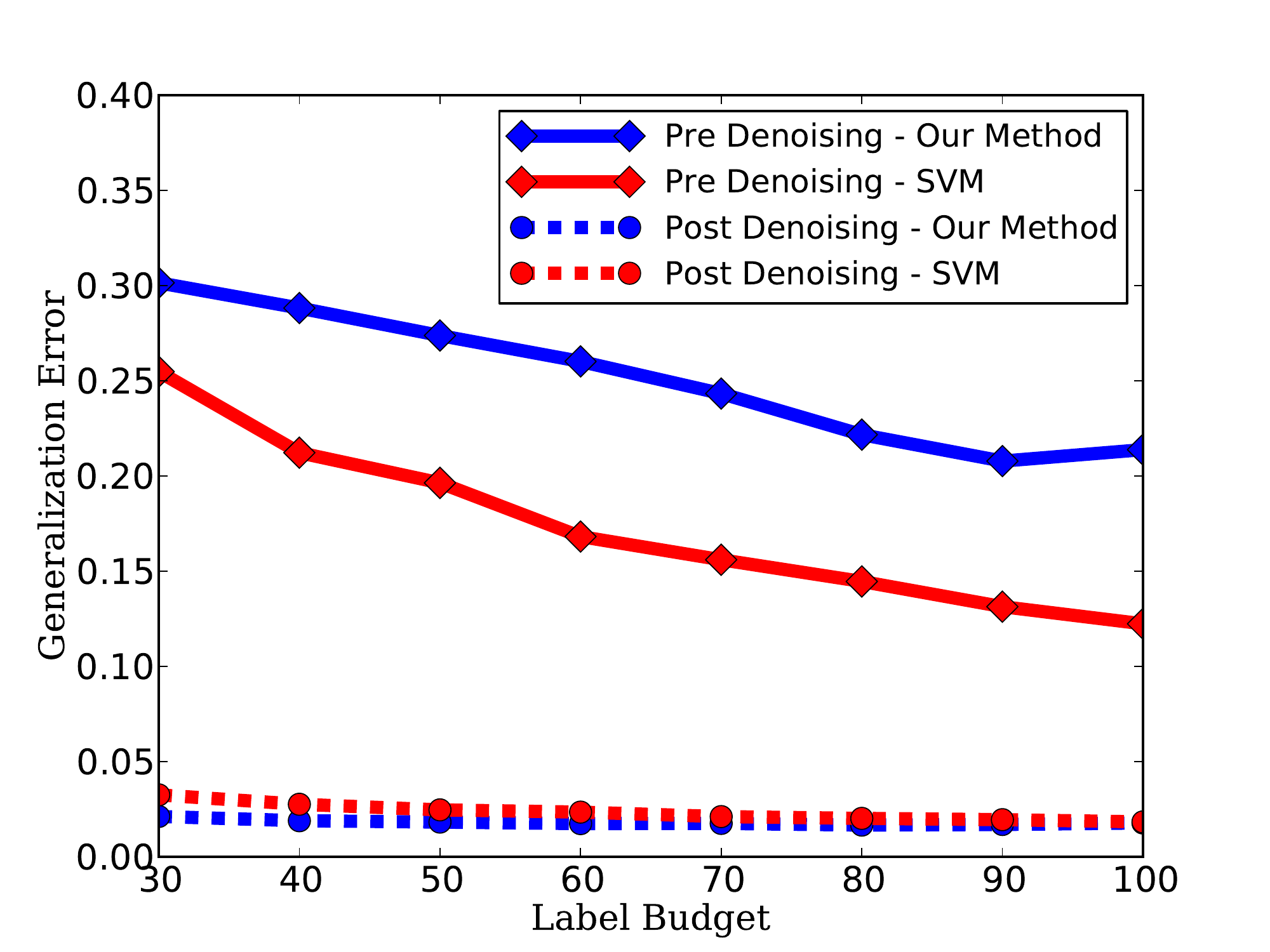}
\includegraphics[width = 0.48\linewidth]{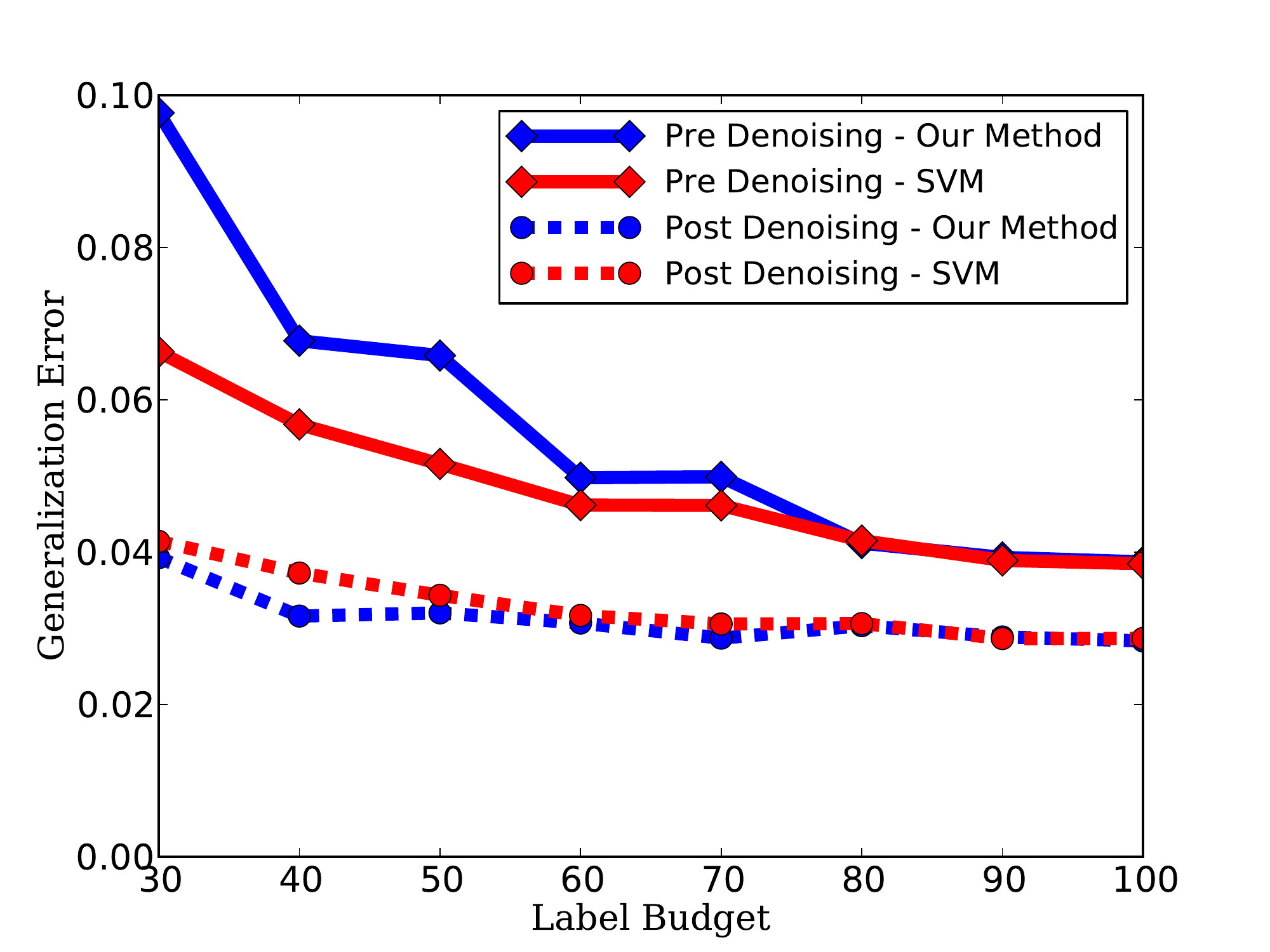}
\caption{Generalization error with 5000 sensors in different noise scenarios. Generalization error in y-axis and Labels used in x-axis. From Left to Right - Random Noise, and Pockets of Noise.\label{fig:5000}}
\end{center}
\end{figure}

\begin{figure}
\begin{center}
\includegraphics[width = 0.48\linewidth]{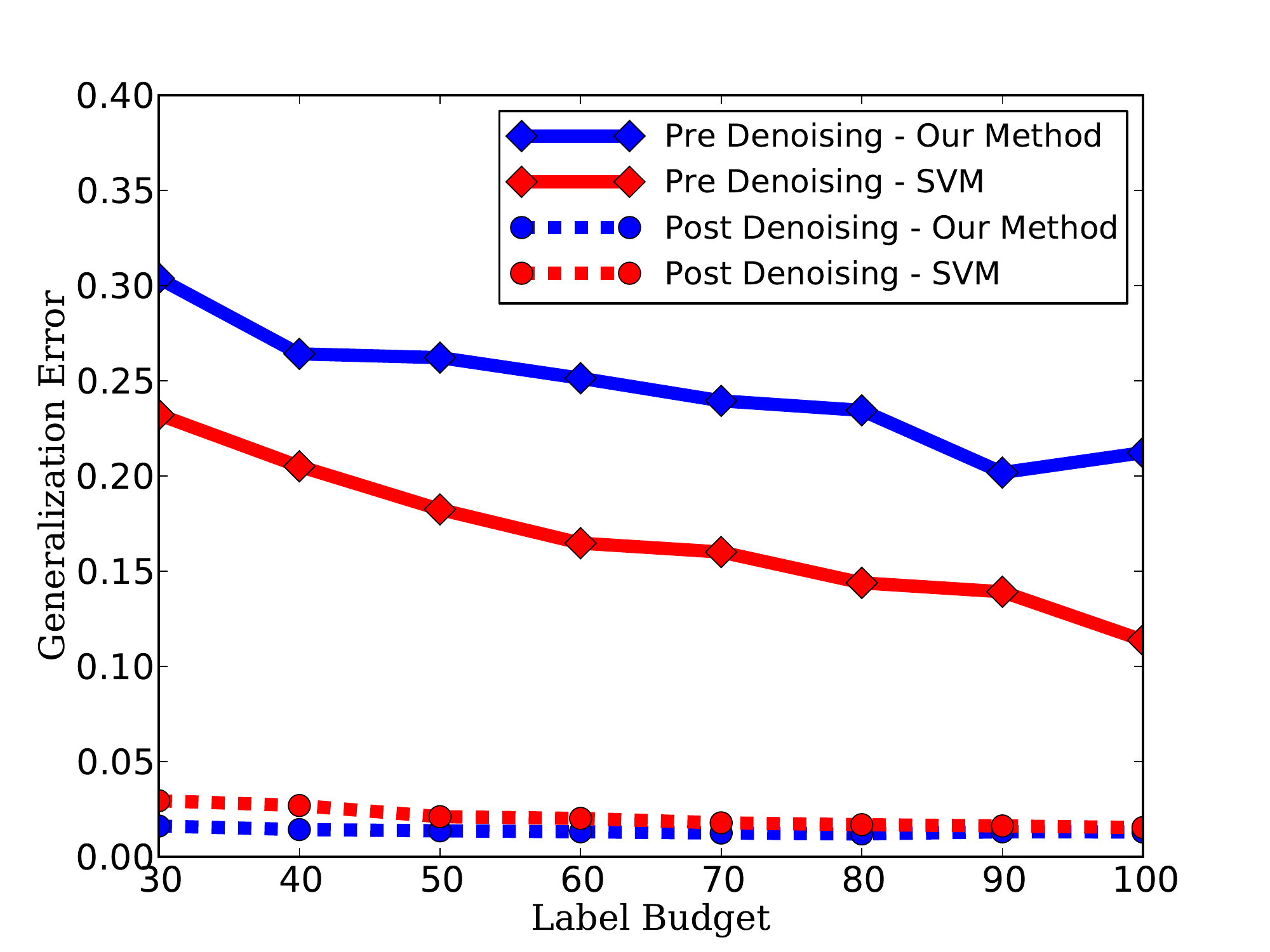}
\includegraphics[width = 0.48\linewidth]{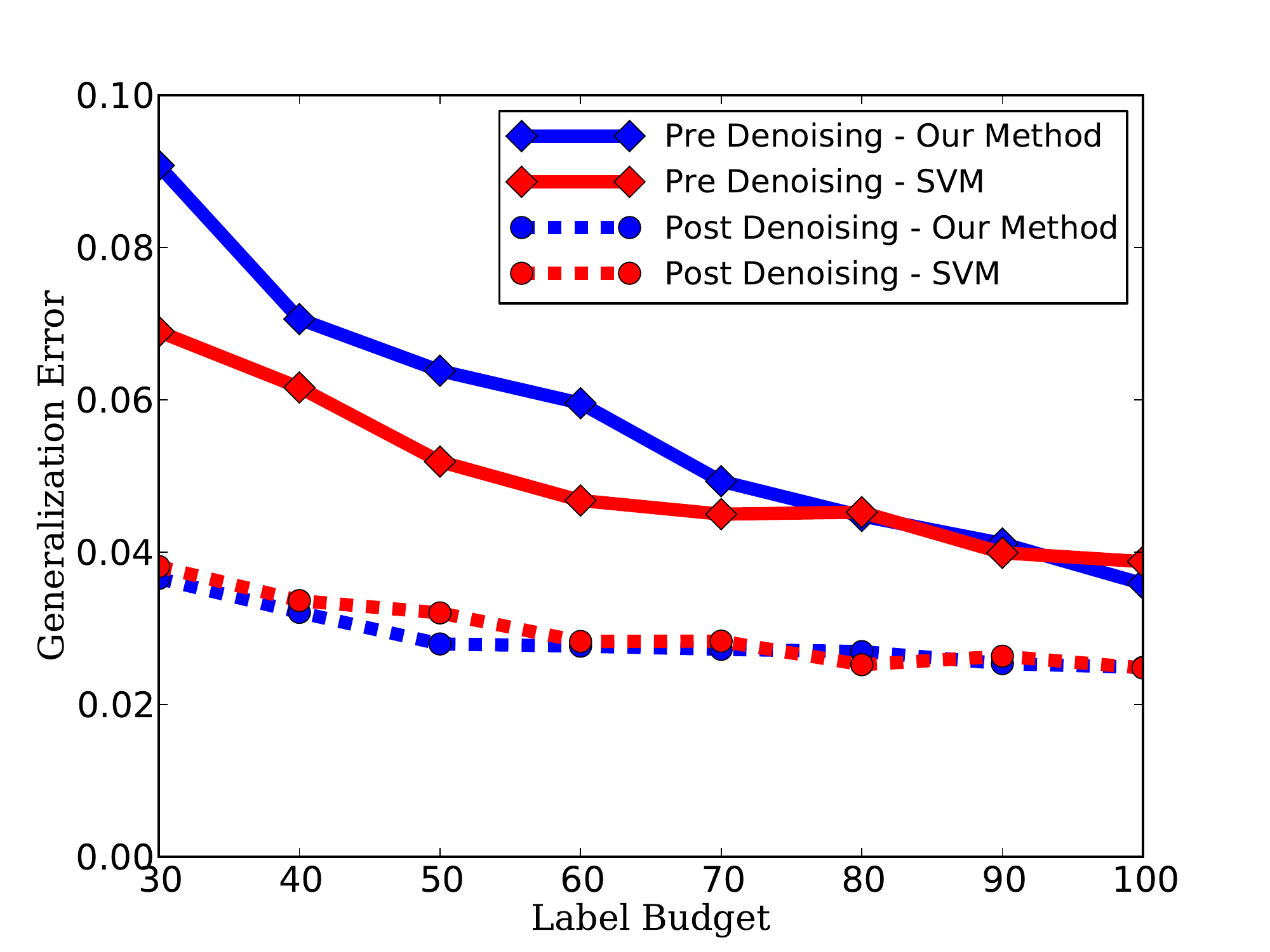}
\caption{Generalization error with 25000 sensors in different noise scenarios. Generalization error in y-axis and Labels used in x-axis. From Left to Right - Random Noise, and Pockets of Noise.\label{fig:25000}}
\end{center}
\end{figure}

\textbf{Effect of communication radius on denoising and learning.}

We analyze here the performance of learning post-denoising as a function of the communication radius for a fixed number of sensors.  In light of Theorem 2 we expect a larger communication radius to improve the effectiveness of denoising.  Figures 8, 9, and 10 show the generalization error pre- and post-denoising for $r \in \{0.2, 0.05, 0.025\}$ with 10000 sensors.  Here denoising helps for both active and passive learning in all scenarios.

\begin{figure}
\begin{center}
\includegraphics[width = 0.48\linewidth]{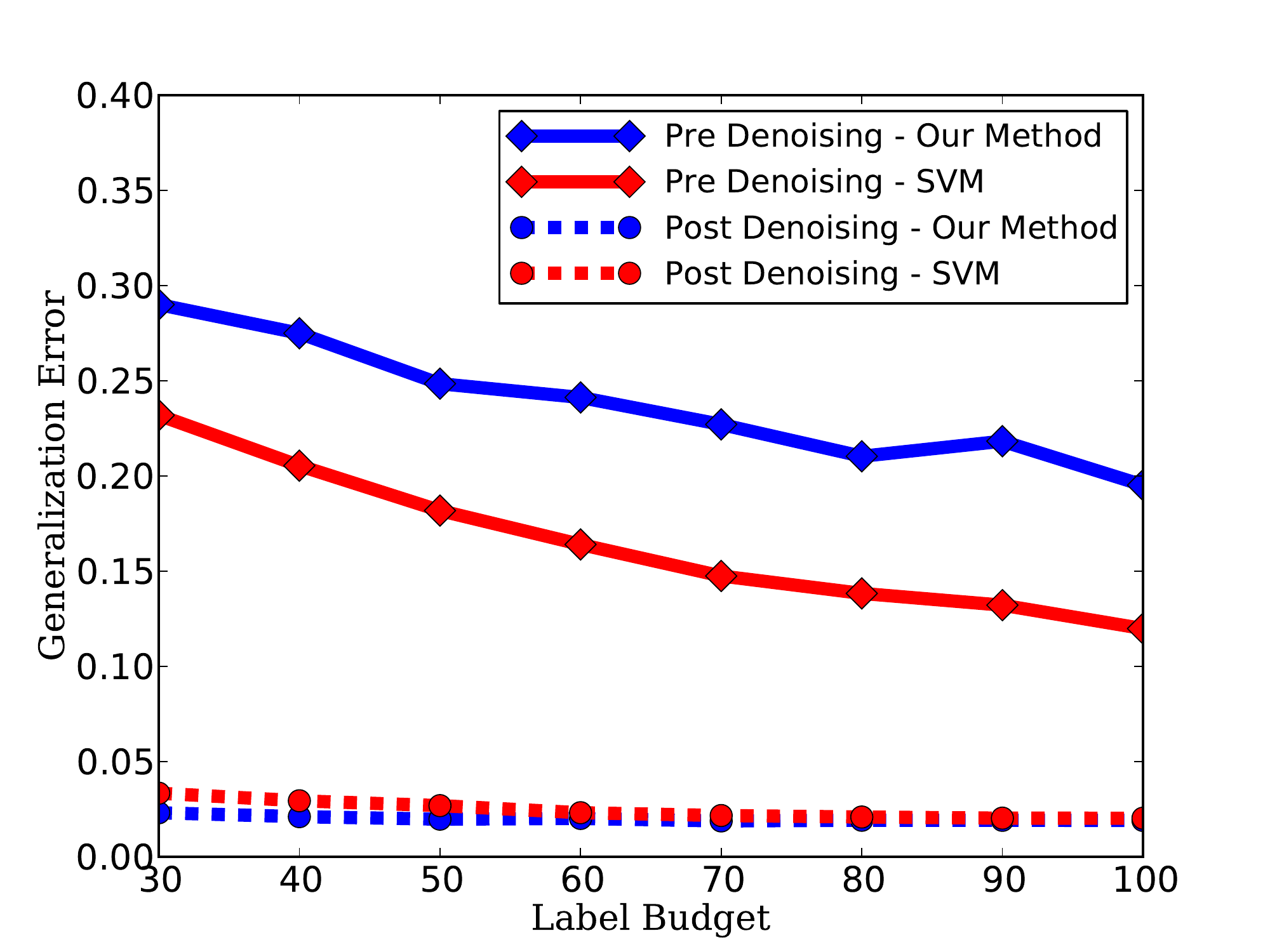}
\includegraphics[width = 0.48\linewidth]{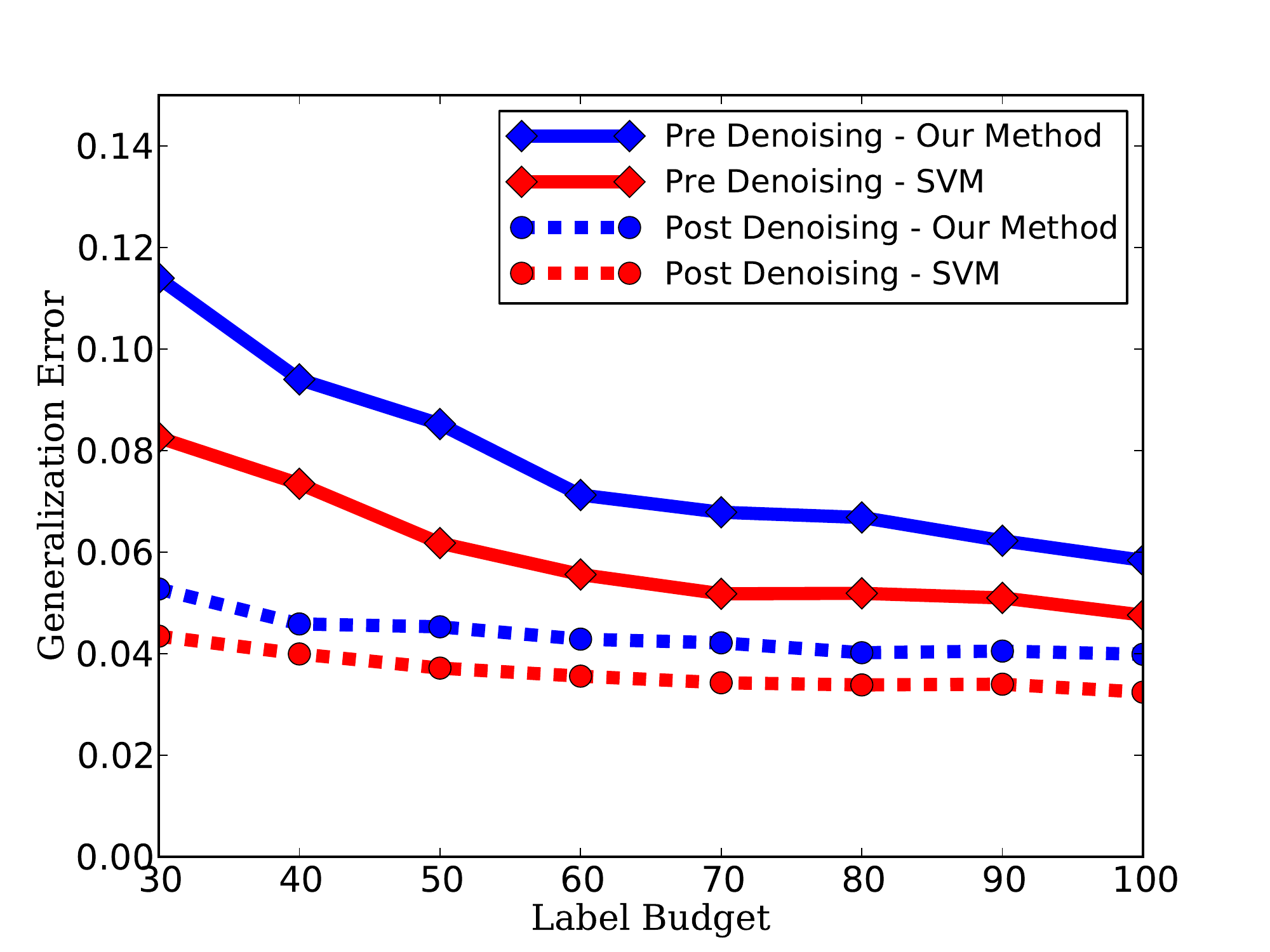}
\caption{Generalization error with connectivity radius of 0.2 and 10,000 sensors in different noise scenarios. Generalization error in y-axis and Labels used in x-axis. From Left to Right - Random Noise, and Pockets of Noise.\label{fig:rad2}}
\end{center}
\end{figure}

\begin{figure}
\begin{center}
\includegraphics[width = 0.48\linewidth]{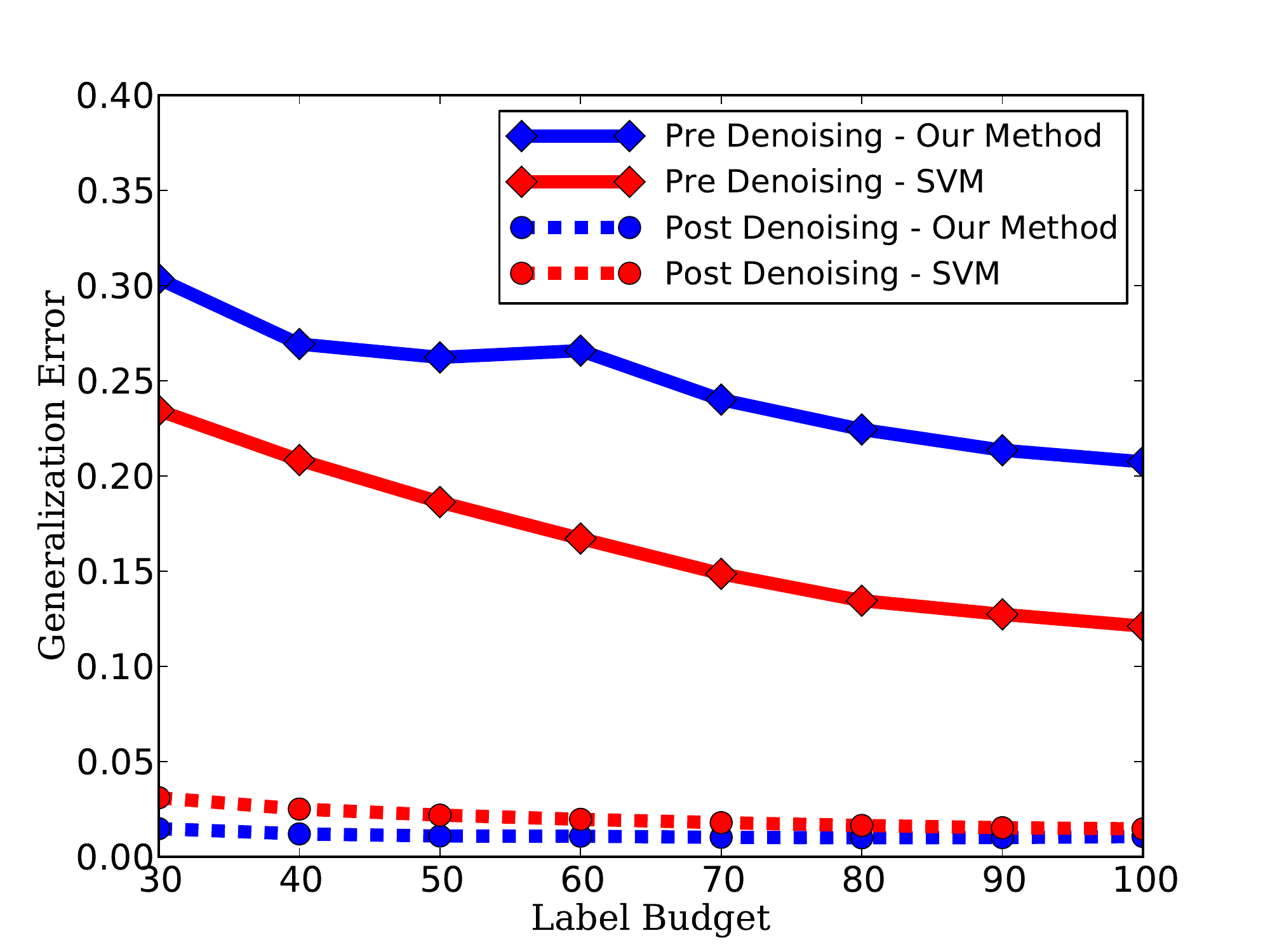}
\includegraphics[width = 0.48\linewidth]{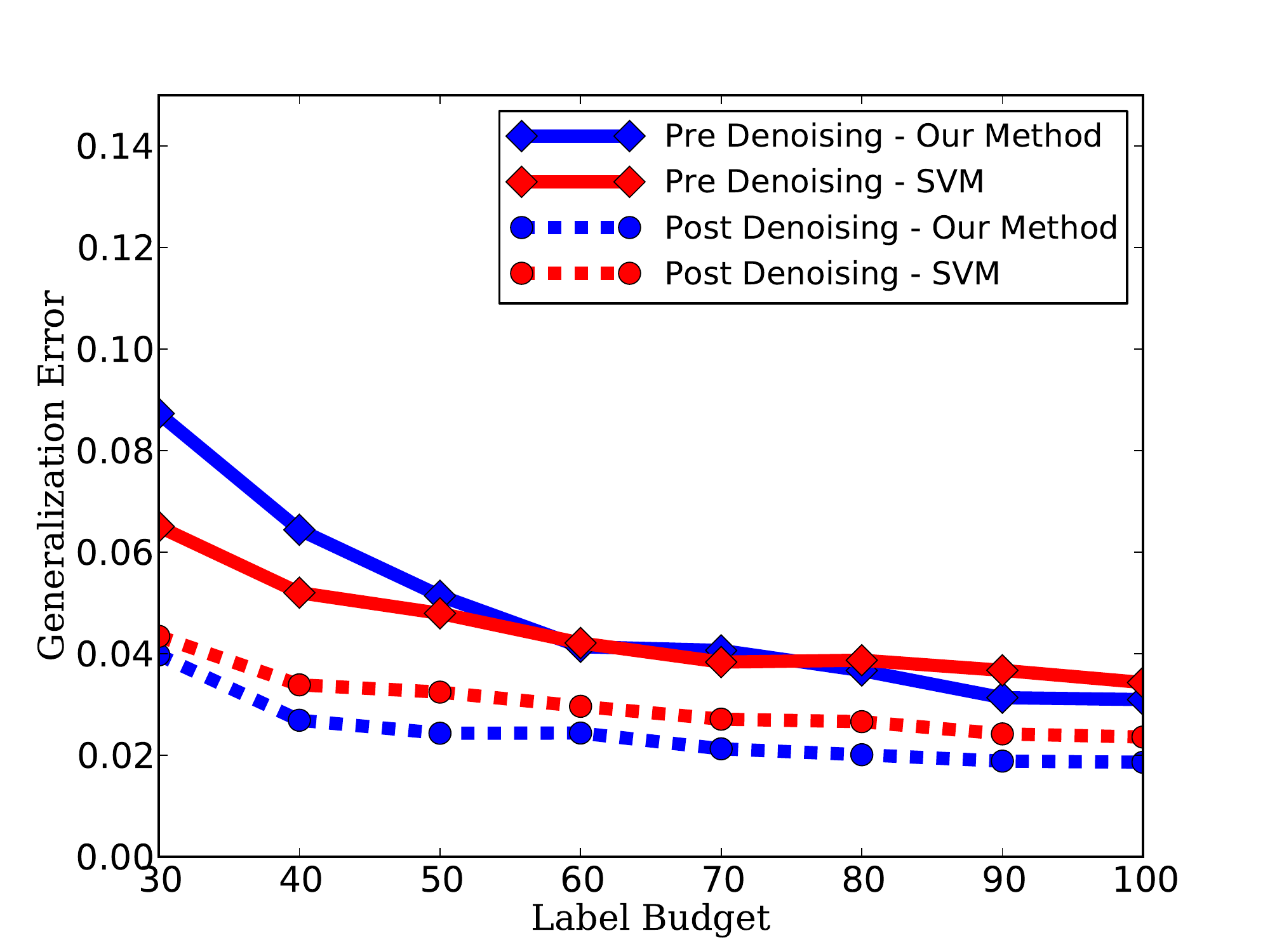}
\caption{Generalization error with connectivity radius of 0.05 and 10,000 sensors in different noise scenarios. Generalization error in y-axis and Labels used in x-axis. From Left to Right - Random Noise, and Pockets of Noise.\label{fig:rad05}}
\end{center}
\end{figure}

\begin{figure}
\begin{center}
\includegraphics[width = 0.48\linewidth]{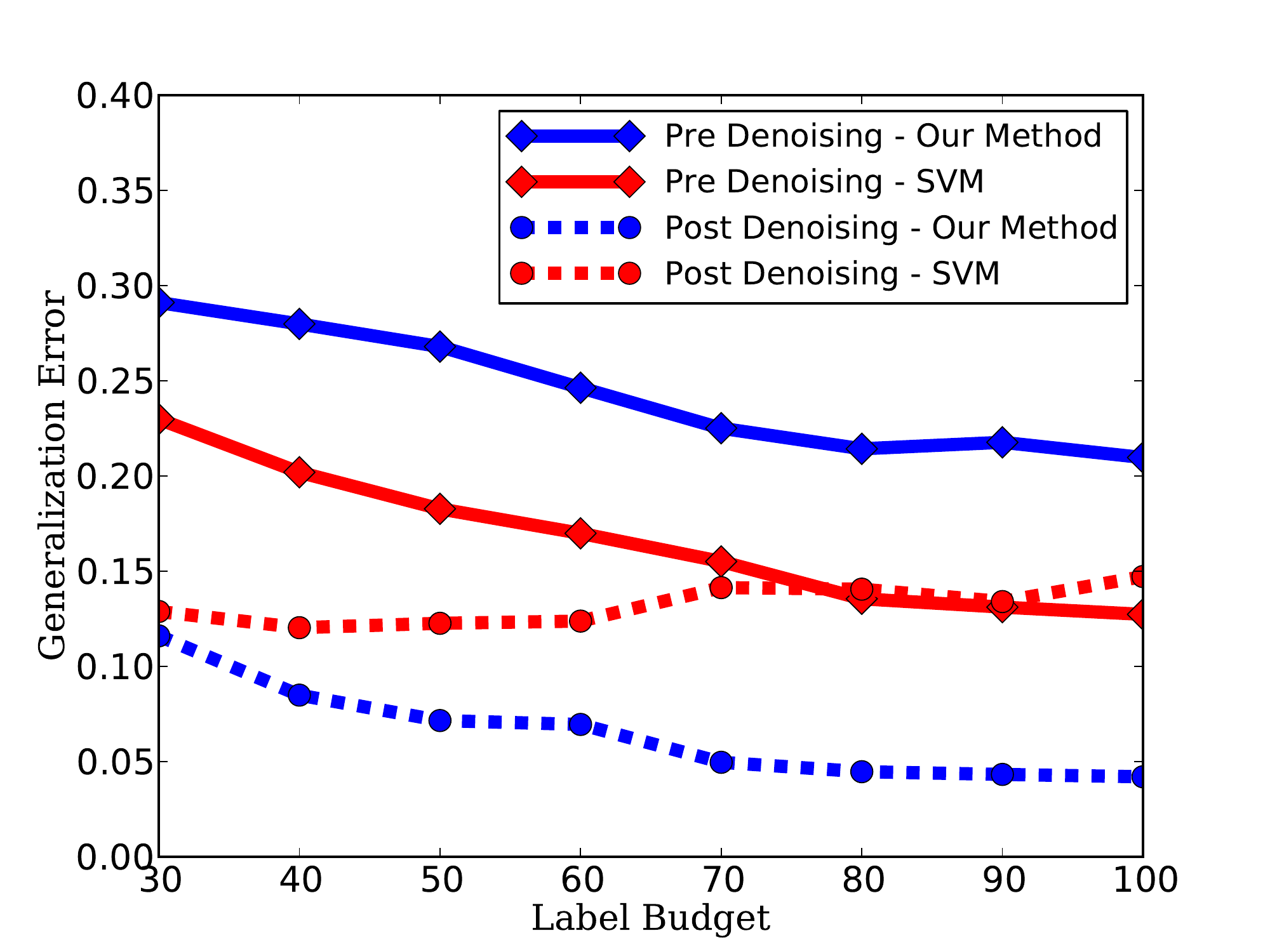}
\includegraphics[width = 0.48\linewidth]{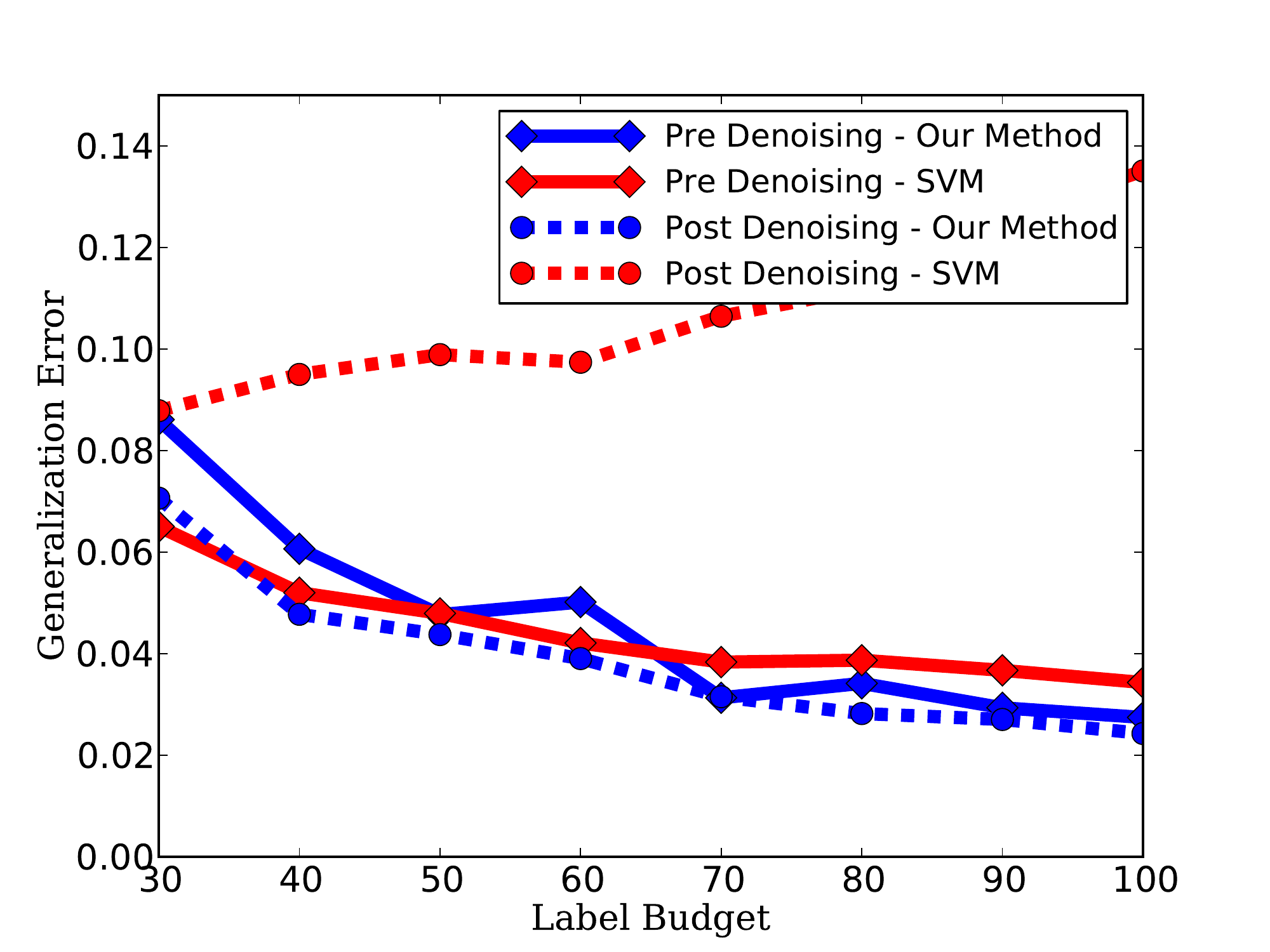}
\caption{Generalization error with connectivity radius of 0.025 and 10,000 sensors in different noise scenarios. Generalization error in y-axis and Labels used in x-axis. From Left to Right - Random Noise, and Pockets of Noise.\label{fig:rad025}}
\end{center}
\end{figure}

\subsection{Kernelized Algorithm Derivation and Results}


\textbf{Derivation of dual with a linear ball constraint}

In order to be able to replace inner products with kernel evaluations in the dual program of the hinge loss minimization step, we replace the ball constraint given by $\|w_k - w_{k+1}\|_2 \leq r_k$ with an equivalent linear constraint $w_k w_{k-1} \geq 1-r_k^2/2$.

\begin{align*}
L =  \sum_{i=1}^{n}{\xi_i}
  - \sum_{i=1}^{n}{\alpha_i (y_i(w_k x_i)}/\tau_k -1 + \xi_i) \\
  + \beta (1 - r_k^2/2 - w_k w_{k-1}) \\ 
  + \gamma( \| w_k \|_2 - 1) - \sum_{i=1}^{n}{\delta_i\xi_i}
\end{align*}

To obtain the dual formulation, we take derivation of the above equation w.r.t to $w_k$ and $\xi$ and substitute these values in the original formulation we obtain


\begin{program} \label{prog:dual_linear_ball}
  \max_{\alpha, \beta, \gamma} \sum_{i=1}^{n}{\alpha_i} - 1/4\gamma\tau_k^2 * \sum_{i=1}^{n}\sum_{j=1}^{n}{\alpha_i \alpha_j y_i y_j x_i x_j} 
  \\ - \beta w_{k-1}/2\tau_k\gamma * \sum_{i=1}^{n}{\alpha_i y_i x_i} 
  \\ - \beta^2 w_{k-1}^2 /4\gamma + \beta (1-r_k^2/2) - \gamma \\ \\
  \text{s.t} \\
  0 \leq \alpha \leq 1  \\
  \beta, \gamma \geq 0
\end{program}

In (\ref{prog:dual_linear_ball}) the lagrangian variable $\gamma$ is present as a denominator in three terms which are negative and as a quantity being subtracted in the objective function. Thus the edge values $0,\inf$ of $\gamma$ would lead to decrease in the objective value and cannot lead to maximum. Thus the maximum value of objective function will be found at $\gamma$ evaluated at $\partial L/\partial \gamma = 0$. Taking the derivative of \ref{prog:dual_linear_ball} w.r.t $\gamma$ we get

\begin{equation}
\gamma = \sqrt{1/\tau_k^2 \sum_{i=1}^{n}\sum_{j=1}^{n}{\alpha_i \alpha_j y_i y_j x_i x_j} \\ 
+ \beta w_{k-1}/2\tau_k + \beta^2 w_{k-1}^2 /4}
\end{equation}

Substituting this value in the (\ref{prog:dual_linear_ball}) and simplifying gives us -

\begin{program} \label{prog:dual_linear_ball_final}
  \min_{\alpha, \beta} \sqrt {\sum_{i=1}^{n}\sum_{j=1}^{n}{\alpha_i \alpha_j y_i y_j x_i x_j} + 2\tau_k \beta w_{k-1} (\sum_{i=1}^{n}{\alpha_i y_i x_i}) + (\beta \tau_k w_{k-1})^2}
  \\ - \tau_k \sum_{i=1}^{n}{\alpha_i} - \tau_k \beta (1-r_k^2/2) \\
  \text{s.t} \\
  0 \leq \alpha \leq 1  \\
  \beta \geq 0
\end{program}

The term under the square root in (\ref{prog:dual_linear_ball_final}) can be simplified as $(\sum_{i=1}^{n}{\alpha_i y_i x_i} + \beta \tau_k w_{k-1})^2$, which further simplifies equation (\ref{prog:dual_linear_ball_final}) to

\begin{program} 
  \max_{\alpha, \beta}  \tau_k\sum_{i=1}^{n}{\alpha_i} + \tau_k\beta (1-r_k^2/2) \\
 - \norm{\sum_{i=1}^{n}{\alpha_i y_i x_i} + \beta \tau_k w_{k-1}} \\
  \text{s.t} \\
  0 \leq \alpha \leq 1  \\
  \beta \geq 0
\end{program}

The resulting optimization objective function can be implemented by expanding out the two norm and value of previous weight vector as $\sum_{l=1}^{p}{\alpha_l y_l x_l}$.

\textbf{Results for kernelized algorithm}

We also test the improvement of the active learning method for non-linear decision boundaries. The target decision boundary 
is a sine curve on the horizontal axis in $R^2$ space, with points above the curve labeled as positive, else negative. Noise was introduced in the true labels through methods described in Section 5.1. For comparison with passive methods we calculate the classification error over 20 trials, where in each trial we average results over 20 iterations. Both the active and passive algorithms use a Gaussian kernel with bandwidth of 0.1 for a smooth estimate of the the boundary. All other parameters remain the same. Results are shown in Figures 6. Notice that the results are similar to experiments with linear decision boundaries.

\begin{figure}
\begin{center}
\includegraphics[width = 0.48\linewidth]{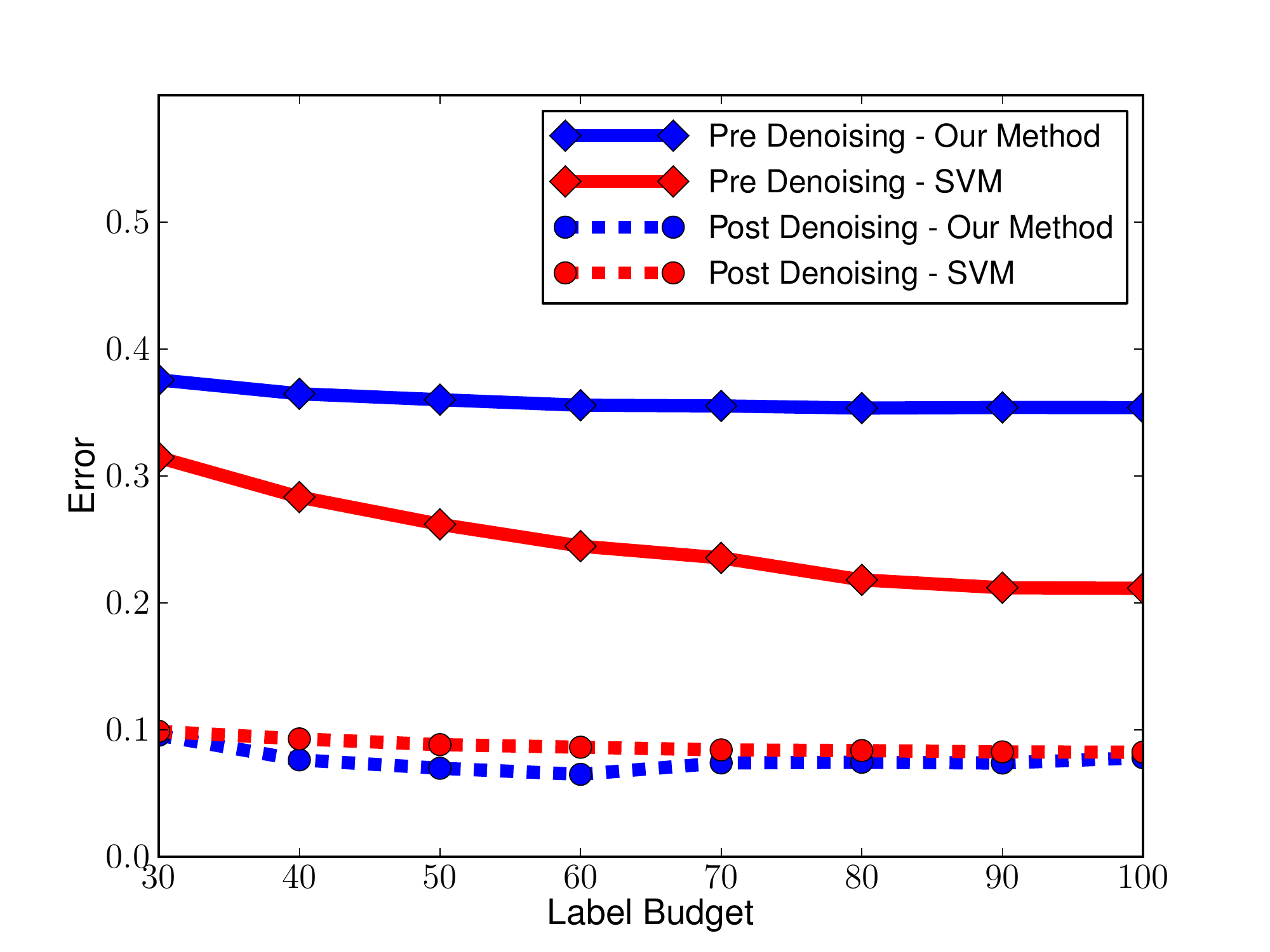}
\includegraphics[width = 0.48\linewidth]{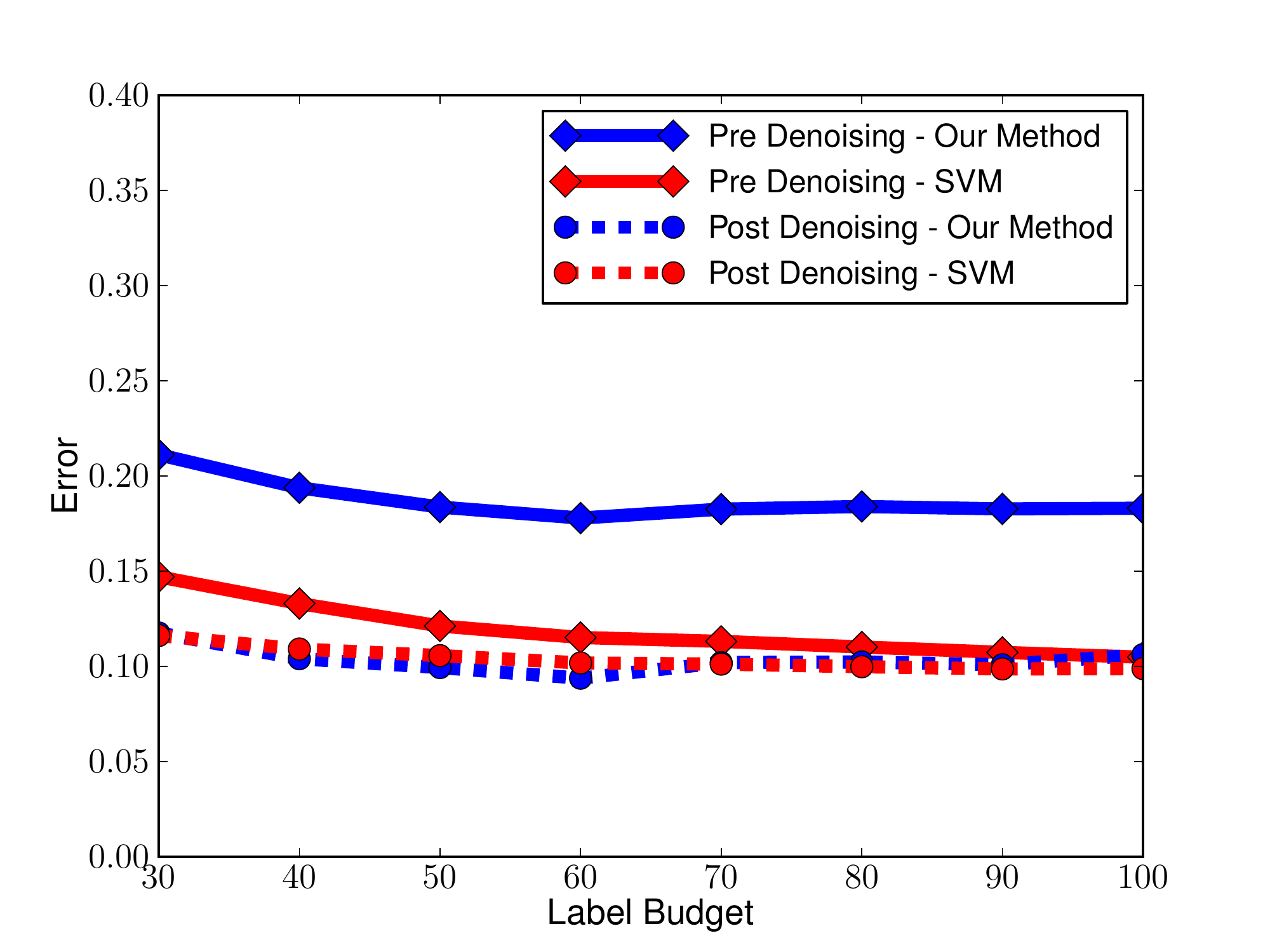}
\caption{Accuracy with Gaussian kernels in different noise scenarios. Accuracy in y-axis and Labels used in x-axis. From Left to Right - Random Noise, and Pockets of Noise.}
\end{center}
\end{figure}

\end{document}